\documentclass[1p]{elsarticle}

\pdfpagewidth=\paperwidth
\pdfpageheight=\paperheight
\usepackage{scalerel}
\usepackage[scr=boondox]{mathalfa}
\usepackage{amsmath,graphicx}
\usepackage{verbatim}
\usepackage{color}
\usepackage{amsfonts,amsbsy,amssymb}
\usepackage{natmove}
\usepackage{enumitem}
\usepackage{natbib}
\setcitestyle{numbers}
\usepackage{upgreek}
\usepackage[table]{xcolor}

\usepackage{algorithm}
\usepackage[noend]{algpseudocode}
\def\bq{{\boldsymbol {\rm q}}}

\def\L{\mathcal L}

\def\d{\!\!{\rm d}}

\def\K{{\mathbb K}}

\def\H{{\mathcal H}}
\def\q{{{\mathscr{q}}}}

\def\bR{\boldsymbol {\rm R}}
\def\tanh{{\rm tanh}}
\def\bsigma{{\boldsymbol \sigma}}
\def\C{{\mathscr C}}
\newcounter{concount}
\newcounter{exampcount}
\newcounter{algcount}

\newtheorem{corollary}{Corollary}

\def\X{{{\rm X}^K}}
\def\XT{{{\rm X}^2}}

\newdefinition{observation}{Observation}
\newtheorem{lemma}{Lemma}
\newtheorem{conjecture}{Conjecture}

\newtheorem{proposition}{Proposition}
\newdefinition{remark}{Remark}
\newdefinition{remarks}{Remarks}
\newdefinition{keyremark}{{\bf Key Remark}}
\newdefinition{definition}{Definition}
\newdefinition{note}{{\bf Note on Terminology}}
\newdefinition{notation}{{\bf On Notation}}
\newdefinition{example}{Example}

\newdefinition{construction}{Construction}
\newproof{proof}{Proof}
\DeclareMathOperator*{\argmax}{argmax}
\DeclareMathOperator*{\argmin}{argmin}
\begin{document}

\author[add1]{Paul N. Patrone}
\ead{paul.patrone@nist.gov}
\author[add1]{Anthony J.\ Kearsley}
\address[add1]{National Institute of Standards and Technology, Applied and Computational Mathematics Division, Gaithersburg, MD 20899, USA}

\renewcommand{\thefootnote}{\fnsymbol{footnote}}

\footnotetext[1]{National Institute of Standards and Technology, Gaithersburg MD 20899, USA}
\footnotetext[2]{Department of Medicine, Division of Infectious Diseases and Immunology, University of Massachusetts Chan Medical School, Worcester,
MA, 01655, USA}

%\title{The Role of Prevalence in Uncertainty Quantification for Machine Learning and Classification Theory}
\title{Probabilistic Consistency in Machine Learning and Its Connection to Uncertainty Quantification}

\begin{abstract}
Machine learning (ML) is often viewed as a powerful data analysis tool that is easy to learn because of its black-box nature.  Yet this very nature also makes it difficult to quantify confidence in predictions extracted from ML models, and more fundamentally, to understand when they admit probabilistic interpretations.  The goal of this paper is to unravel these issues and their connections to uncertainty quantification (UQ) by deriving a \textit{level-set theory of classification} that establishes an equivalence between certain types of self-consistent ML models and class-conditional probability distributions.  We begin by studying the properties of binary Bayes classifiers, recognizing that their boundary sets can be reinterpreted as level-sets of density ratios quantifying  the relative probability that a sample point belongs to a given class.  By promoting the prevalence (fraction of elements in a class) to the role of an \textit{affine parameter} that orders these level sets, we show that Bayes classifiers satisfy important monotonicity and class-switching properties that can be used to extract the density ratios without direct access to the boundary sets.  Moreover, this information is sufficient for tasks such as constructing the multiclass Bayes classifier from its pairwise counterparts and estimating inherent uncertainty in the class assignments.  In the multiclass case, we use these results to derive normalization and self-consistency properties of Bayes classifiers, the latter being equivalent to the law of total probability.  We then show how these properties  equip {\it arbitrary} ML models with valid probabilistic interpretations arising from inherent class conditional probability distributions.  Throughout, we demonstrate how this analysis informs the broader task of UQ for ML via an uncertainty propagation framework.                           
\end{abstract}

\maketitle

\section{Introduction}
\label{sec:introduction}

From our perspective, the formulation of a complete uncertainty quantification (UQ) framework for machine learning (ML) remains an unresolved problem.  Several works have addressed empirical aspects of this problem through the lens of Monte Carlo methods and related sampling techniques \cite{UQ_Survey1,UQ_Survey2,UQ_Survey3,UQ_Survey4}.  Moreover, it is well known that classification theory can be grounded in probability \cite{MLBook}.  Indeed, this perspective has led to important observations, such as the recognition that the Bayes classifier, which minimizes the expected classification error, also arises from optimization of surrogate loss functions \cite{Hconsistency1,Hconsistency2,Hconsistency3,Hconsistency4}.  In this context, many authors have addressed elements of uncertainty estimation as a task in \textit{calibration}, i.e.\ ensuring that probabilistic ML outputs are consistent with distributional information about data \cite{CalibrationReview,Calibration1,Calibration2,CalibrationApplied,Calibration2016,Calibration2017,Calibration2018}.  However, these works have not considered more fundamental and perhaps subtle questions: (i) are ML predictions always mathematically consistent with probability theory; (ii) can one meaningfully associate aleatoric uncertainty with arbitrary ML models \cite{SmithUQ}; and thus, (iii) do they even permit meaningful UQ?  Stated differently, when is calibration possible, and what mathematical properties of classifiers should inform calibration?

A key difficulty in answering these questions arises from the fact that it is not always clear how, or even \textit{if}, arbitrary ML models are abstract representations of training data.  For example, a discriminative Bayes classifier by definition expresses information about the underlying conditional probabilities $\Pr[r|C]$ of ``inputs'' $r$ generated by elements from class $C$.   But while such classifiers provide minimum error guarantees in a global sense,\footnote{That is, by minimizing the expected classification error.} they generally do not yield \textit{a priori} pointwise estimates of their confidence.  This begs the question of whether discriminative classifiers contain  all of the information about the distributions $\Pr[r|C]$, or if something has been lost.  Such observations also lead us to speculate that if an arbitrary classifier is to admit meaningful UQ, then it must be ``fully equivalent'' to the underlying conditional distributions $\Pr[r|C]$.

The purpose of our manuscript is to understand in what sense this equivalence between classifiers and distributions holds.  The main idea of our analysis is to first determine which properties of Bayes classifiers arise from their probabilistic structure and then demonstrate how these properties alone enable one to associate an arbitrary classifier $\hat C$ with distributions $\Pr[r|C]$.  To achieve this in practice, we consider the boundary sets of Bayes classifiers, which are reinterpreted as level sets of pairwise density ratios quantifying the likelihood that an input $r$ arises from a given class.  By promoting prevalence (fraction of elements in a class) to the role of an affine parameter that orders these level sets, we arrive at an important montonicity property that holds in general for Bayes classifiers.  We then demonstrate how this property can be used to deduce the density ratios, and hence information about $\Pr[r|C]$, in terms of the affine prevalence value at which the class changes.  In extending this result to the multiclass setting, we also identify normalization and self-consistency criteria, the latter expressing the law of total probability for classifiers.  By recourse to the density ratios, our analysis then: (i) demonstrates the sense in which monotonicity, self-consistency, and normalizability enable one to establish a correspondence between arbitrary $\hat C$ and $\Pr[r|C]$; and (ii) shows how the aforementioned analyses can account for distribution shifts due to changes in the prevalence of a test population.  %Importantly, these results permit a deeper analysis of UQ for ML via a formal uncertainty propagation framework and thereby facilitate practical tasks such as estimation of inherent uncertainty (a type of aleatoric uncertainty) in the class labels.

%analyze this claim by deriving a level-set theory of classification that establishes necessary conditions for there to be a mathematical correspondence between classifiers $\hat C$ and distributions $\Pr[r|C]$.    The main idea of this analysis is determine which properties of Bayes classifiers arise from its probabilistic structure and then demonstrate how these properties lead to the desired correspondence.  

Because probability is central to our analysis, a related goal of this manuscript is to unravel the role of uncertainty quantification (UQ) as a tool for understanding ML more generally, not just assessing confidence in its predictions.    In other words, we take the position that UQ is inherent to ML, and to analyze ML is to study the theory of UQ for classification.  To realize this in practice, we adopt a small but important change in perspective.  Canonical approaches to ML often directly consider the conditional probability $\Pr[C|r]$ of class $C$ conditioned on the input $r$ \cite{CrossEntropy1,UQ_Survey1,UQ_Survey2,UQ_Survey3,UQ_Survey4,SandiaUQML,Langford05}.  But it is well known that
\begin{align}
\Pr[C|r] = \frac{\Pr[r|C]\Pr[C]}{\Pr[r]} = \frac{\Pr[r|C]\Pr[C]}{\sum_{C'} \Pr[r|C']\Pr[C']}, \label{eq:condprob}
\end{align}
where $\Pr[C]$ is the prevalence of class $C$.  Equation \eqref{eq:condprob} is sometimes used to motivate the definition of a generative classifier as one that directly models $\Pr[r|C]$ \cite{Disc3}.  But Eq.\ \eqref{eq:condprob} demonstrates that this distinction is superficial, as it simply amounts to whether $\Pr[r|C]$ is implicit or explicit in the modeling.  Moreover, it is well known that one can objectively estimate $\Pr[C]$ without recourse to classification or Bayesian priors \textit{per se} \cite{QuantificationBook,Quantification71,Quantification2005,Quantification2009,Quantification2016,Quantification2017,Quantification2018,Quantification2022,Quantification2024}.   In this context, we argue that the study of ML  is more natural when reformulated in terms of $\Pr[r|C]$, since it isolates and makes explicit the way in which classifiers represent data.  This shift in perspective is central to our analysis, and it immediately points to three thematic elements of our work.
  
First and most simply, Eq.\ \eqref{eq:condprob} allows one to better identify and study individual sources of uncertainty.  Note, for example, that even when $\Pr[r|C]$ and $\Pr[C]$ are known exactly, it is possible that $0 < \Pr[C|r] < 1$, in which case $r$ cannot unambiguously be assigned a single class.  We refer to this effect as ``inherent'' uncertainty because it is a fundamental property of the underlying input space \cite{UQ_Survey2}.  In contrast,  the act of choosing a hypothesis set $\H$ of classifiers is equivalent to selecting a family of distributions $\Pr[r|C]$ for regresssion.  This choice may not be well adapted to the data and thereby introduce additional uncertainty into $\Pr[r|C]$, in this case associated with modeling \textit{per se}.  In classification theory, these issues are sometimes acknowledged in the context of $H$-consistency bounds \cite{Hconsistency1,Hconsistency2,Hconsistency3,CrossEntropy1}, but they  have not been fully separated from estimates of $\Pr[C]$, which is yet  a third and often significant source of uncertainty.  Thus, Eq.\ \eqref{eq:condprob} allows us to disentangle such effects via a formal uncertainty propagation framework, which quantifies their individual and compounding contributions to the total uncertainty budget.  

Second, Eq.\ \eqref{eq:condprob} highlights a key distinction between $\Pr[C]$ and $\Pr[r|C]$: the former is a property of a population, whereas the latter is a property of data and the input space $\Gamma$.  As such, the prevalence can (and often does) vary independently of conditional distributions.  This motivates our second interpretation of $\Pr[C]$ as a free-parameter, i.e.\ an \textit{affine prevalence} that ``deforms'' one classifier into another.  In this way, we promote the prevalence from the role of a Bayesian prior -- in essence, a source of uncertainty -- to that of a control variable that parameterizes the loss function used in training.  This also clarifies our claim that analysis of ML is tantamount to studying UQ of classification: it is useful to distinguish uncertainty  inherent to the data from that associated with a population.  As a result, this distinction allows us to argue that aleatoric uncertainty is entirely controlled by $\Pr[r|C]$, thereby justifying the need to understand \textit{when} such distributions are induced by a classifier.

Third, Eq.\ \eqref{eq:condprob} highlights a tension between local and global structure of ML algorithms.  In general, classifiers are trained by optimizing  functions such as the {\it average} classification accuracy \cite{MLBook}.  But Eq.\ \eqref{eq:condprob} makes it clear that we are interested in \textit{pointwise} predictions, since one is always given specific instances of $r$ to analyze.  Thus, our main task can be understood as deducing the pointwise structure of $\Pr[r|C]$ from the global properties of its representation as a classifier.  The affine prevalence plays a fundamental role in making this connection, since we vary it continuously to map out level sets of these conditional distributions.  But we emphasize that this is also a strained process: we encounter issues associated with sets of zero-measure and indeterminate forms, which require auxiliary assumptions to make the global-local connection rigorous.  {\it The reader should remain aware that in practical settings, this tension is a fundamental challenge in UQ for ML; see Sec.\ \ref{sec:examps} in particular.}

It is also important to note that a perspective based on Eq.\ \eqref{eq:condprob} benefits from a subtle reformulation of many well-known concepts in classification theory.  For example, we leverage Bayes-optimal classifiers in detail, but the standard way of constructing them obscures the monotone-class structure parameterized by the affine prevalence.  To make this property apparent, it is useful to recast Bayes-optimal classifiers in set-theoretic terms. This reformulation has the added benefit of clarifying certain challenges posed by sets of measure zero.  

An important limitation of this work is the fact that we primarily consider what amounts to an \textit{infinite-sample} limit; see Ref.\ \cite{Zhang04} for related ideas.  This corresponds to a setting in which the probability distributions $\Pr[r|C]$ can be determined exactly and for which inherent uncertainty is the only source of ambiguity in the true class labels.  While this situation never holds in practice, it does represent a best-case scenario for classification uncertainty, and all other sources build upon this.  Moreover, we feel that the complexity of this subject warrants a separate analysis of each component of the UQ, for which the inherent uncertainty is perhaps the most fundamental part.  A full study of UQ for ML is beyond the scope of this (or even likely one single) manuscript.  Several works in progress address additional elements of the UQ for ML from the perspective of Eq.\ \eqref{eq:condprob}.

Finally, a historical note is in order.  A variation on the monotonicity property was previously noticed as far back as 2005 by Langford and Zadrozny \cite{Langford05}, primarily in the context of binary classifiers. To a lesser extent Ref.\ \cite{Halck} examined similar ideas, and Ref. \cite{MultiGen} considered some extensions to the multiclass setting.  However, our analysis more thoroughly examines the properties of the Bayes classifiers and thereby discovers that the concept of monotone classifier requires significant refinement and generalization.  This is necessary to deduce the self-consistency and normalizability criteria, which appear to have been acknowledged but not fully understood in previous works; see, e.g.\ the Discussion section in Ref.\ \cite{Halck}.  In a related vein, extensions to the multiclass setting require extreme care due to issues associated with non-overlapping supports of the conditional distributions.  In fact, this leads to a fundamental issue of how to interpret classifiers that are evaluated on points outside of their domain of action, and we show that this is analogous to trying to evaluate indeterminate forms such as $0/0$.  Thus, our analysis is partly meant to address unresolved mathematical questions in previous work.     

The rest of this manuscript is organized as follows.  Section \ref{sec:global} gives a broad overview of our perspective on the relationship between classification and probability (\ref{subsec:overview}), which allows us to ground the present manuscript in the broader UQ literature and identify sources of uncertainty in ML models (\ref{subsec:imp4UQ}).  Section \ref{sec:context} provides  background theory for the binary setting.  Section \ref{sec:level-set} presents our main results, which fully generalize  the level-set and class-switching theorems to an arbitrary multi-class settings.  Section \ref{sec:examps} provides examples of our main results and uses these to highlight practical issues of determining when a classifier has a probabalistic interpretation.  Section \ref{sec:discussion} considers our results more broadly in the context of ML theory and discusses limitations and open directions.  

\subsection{On the Conventions Used Herein}

A key conceptual challenge of our analysis is that quantities such as the prevalence $\Pr[C]$ frequently change roles.  For example, $\Pr[C]$ is sometimes an independent variable and at other times a fixed parameter.  To fully account for these changes, our exposition would require an overwhelming amount of notation, which we wish to avoid.  In an attempt to balance simplicity with precision, we therefore adopt the following conventions.  

\begin{itemize}
\item We do not distinguish a function from its action on its argument.  For example, if $\Gamma$ and $\K$ are sets, $r\in \Gamma$, and $\hat C:\Gamma \to \K$, then we refer to both $\hat C$ and $\hat C(r)$ as a function.  
\item A semicolon in the argument of a function is used to distinguish between variables and fixed parameters.  For example, $Q(r;q)$ should be understood as a function of $r$ with $q$ fixed.   More exotic arguments such as sets and partitions also sometimes appear behind semicolons and should likewise be understood as fixed.  
\item We sometimes encounter situations in which the roles of two arguments switch back and forth, e.g.\ one is fixed and the other is variable.  In such cases, we simply use a comma to separate arguments and indicate in writing or through context which argument is fixed.  Thus, the notation $\hat C(r,q)$ indicates a function of $r$ and $q$, but for which either $r$ or $q$ (but not usually both) may be fixed.
\item We often encounter objects such as sets $D$ that are parameterized by scalars or vectors $q$.  For convenience, we denote this dependence using the aforementioned ``function'' notation; e.g.\ $D(q)$ is a set that is parameterized in some way by $q$. 
\item When referring to the general concept of prevalence (i.e.\ divorced from any particular setting), we typically use the notation $\Pr[C]$.  When referring the prevalence of a population, we use the symbol $\chi$.  When referring to the prevalence as a parameter, we use the symbol $q$.  
\item The symbol $\mathbb I$ denotes the indicator function. 
\item We use the extended number system in which $\infty$ is a valid number (not just a limit) satisfying the properties that: $a\cdot \infty = \infty$ for $a$ positive and bounded; $a/0 = \infty$ for $a$ positive and bounded; the product $\infty \cdot 0$ (when interpreted as $\infty \cdot 0^+$) is a non-negative indeterminate number; \textcolor{black}{and the ratio $0/0$ (when interpreted as $0^+/0^+$) is a non-negative, indeterminate number}.  We also use the notation $[0,\infty]$ to denote the interval containing its endpoints, including $\infty$.  See also the preface of Ref.\ \cite{Tao}, for example.  
\end{itemize}

\section{Global Perspective on UQ for ML}
\label{sec:global}

\subsection{Overview of our perspective: what is a classifier?}
\label{subsec:overview}

The overarching goal of this manuscript is to establish a fundamental connection between classification, probability, and prevalence.  However, canonical definitions of classifiers do not provide sufficient mathematical structure to make this connection rigorous, so that it is necessary to first revisit the simpler question, ``what is a classifier?''

%The question, ``what is a classifier?'' has been well studied in the ML community.  However, canonical answers do not provide sufficient mathematical structure for our main goals.  We therefore begin by revisiting this question in the context of probability, with the aim of providing a global perspective on our analysis.    

To fully answer this question, it is necessary to examine the objects of classification {\it per se}.  Consider, then, a setting with a \textit{test population} or sample space $\Omega$ whose members $\omega \in \Omega$ belong to one and only one of $K$ classes.  The true class $C(\omega)$ of an individual $\omega$ is a discrete random variable, and without loss of generality, we may assume $C(\omega) \in \{1,2,...,K\}$.  We refer to elements $\omega$ as \textit{test points} or \textit{test samples}.  

In practice, the true class of a test point is unknown, and instead, this quantity is to be estimated via classification.  To facilitate this, we are often  given a different random variable $r(\omega)\in \Gamma$ for some set $\Gamma$.  Regarding $\Gamma$, we require only a few assumptions, mainly that there exist density functions that quantify the probability of a measurement outcome $r(\omega)$ conditioned on the class $C(\omega)$.  As $\Omega$ is the test population, $\Gamma$ is therefore the set of possible ``measurement'' outcomes for a point $\omega \in \Omega$.

It is important to note that the structure of $\Gamma$ and its underlying distributions do not fully specify the relevant properties of $\Omega$.  A test population always has the additional and independent property that it can be parameterized by a prevalence vector $\chi=(\chi_1,...,\chi_K)$, where ${\chi_k = \Pr[C(\omega)=k]}$.  For later clarity, we refer to such a $\chi$ as a \textbf{test prevalence} and express its relationship to the test population via the notation $\Omega(\chi)$.  Clearly $\chi$ is an element of the $K$-dimensional probability simplex
\begin{align}
\X =\{q: q\in \mathbb R^K, \sum_{k=1}^K q_k=1, q_k \ge 0 \}. \label{eq:prob_simplex}  
\end{align}
We emphasize that this interpretation of prevalence should always be treated as a fixed property (i.e.\ parameter) of $\Omega$.  Oftentimes the test prevalence is unknown \textit{a priori}.

%\begin{definition}[Test Prevalence]\label{def:testprev}
%Let $\omega$ be a sample point with $C(\omega) \in \K$.  We define the {\bf generalized test prevalence} (or more simply, the test prevalence) to be the vector $q=(q_1,q_2,...,q_K)$, where $q_j$ is the probability that $C(\omega)=j$.  That is, $\Pr[C(\omega)=j]=q_j$.   
%\end{definition}

In this context, an arbitrary classifier is often defined as a mapping 
\begin{align}
\hat C:\Gamma \to \K,  \label{eq:oldclassifier}
\end{align}
where $\K = \{1,2,...,K\}$ is a discrete set \cite{MLBookCh1}.
Traditionally, one treats $\hat C(r(\omega))$  as a proxy for the true class $C(\omega)$ because the former can typically be computed, whereas the latter cannot.  In contrast, we eschew this perspective because the act of classifying data is subjective, whereas our interest is in understanding if and how a classifier contains \textit{objective} information about data, i.e.\ its underlying probability distributions.

To elaborate on this point, consider that Eq.\ \eqref{eq:condprob} quantifies all of the information available for making a decision, given a fixed space $\Gamma$ and the probabilities $\Pr[r|C]$ and $\Pr[C]$.  How one chooses to use that information depends on the application at hand; e.g.\ a Bayes optimal classifier may not be suitable for settings that heavily penalize one type of misclassification over another.  But in light of Eq.\ \eqref{eq:condprob}, one could argue that $\Pr[r|C]$ and $\Pr[C]$ suffice for constructing ``reasonable'' classifiers and estimating their induced uncertainties.  In this context, our analysis ultimately seeks to answer the question: given a classifier that is sufficiently ``adapted'' to the data, can we deduce $\Pr[r|C]$?  Were this possible, it would show that the act of training a classifier is equivalent to statistical regression, i.e.\ a modeling exercise.  But more importantly, it would also allow one to modify a classifier to account for distribution shifts (e.g.\ in the prevalence $\Pr[C]$) and alternative notions of loss.  Thus, the main task of our manuscript can be restated as determining the sense in which the function $\hat C$ represents the probability $\Pr[r|C]$

%In this context, our analysis ultimately seeks to answer two questions: (i) given a Bayes-optimal classifier, can we extract the true underlying $\Pr[r|C]$; and (ii) given an arbitrary classifier $\hat C$, is there a natural way to associate with it distributions $\Pr[r|C]$?  In other words, how do classifiers $\hat C$ represent probabilistic information.  

%In answering these questions, it would be possible to not only estimate inherent uncertainty in the class labels, but also modify a classifier to account for distributions shifts arising from changing $\Pr[C]$.  Moreover, question (ii) would formally show that training a classifier is equivalent to statistical regression.  

%In this light, .  

To understand why this task is difficult, it is instructive to examine the ways in which classifiers in the spirit of Eq.\ \eqref{eq:oldclassifier} fail to represent probabilities.  Under the best of circumstances, we would want to construct $\hat C$  such that 
\begin{align}
\hat C(r(\omega)) = C(\omega) \label{eq:idealized}
\end{align}
for every $\omega$.  Equation \eqref{eq:idealized} corresponds to a ``gold-standard'' classifier for which there is no uncertainty in the class assignments.   This is not necessarily unattainable.  It is straightforward to show that when the supports of $\Pr[r|C]$ are disjoint, there exists a $\hat C$ satisfying Eq. \eqref{eq:idealized}; see, e.g., Defs.\ \ref{def:inherent} and \ref{def:partition_induced}, as well as the concept of linearly separable populations in the context of support-vector machines \cite{SVM1,SVM2,SVM3}.  But for any value of $r$ at which the $\Pr[r|C]$ overlap, it is possible that  $C(\omega)$ is random but $\hat C(r(\omega))$ is deterministic.  Thus, the breakdown in Eq.\ \eqref{eq:idealized} arises from the fact that many classifiers are \textit{functions}, i.e.\ mappings from $r$ to one and only one value of $k\in \K$.  

We can nominally overcome this problem by equipping $\hat C$ with additional structure.  To achieve this, we could posit a second sample space $\Sigma$ with points $\varsigma \in \Sigma$ and  classify $\omega'\in\Omega' = \Omega \times \Sigma$.  In this case, it would seem reasonable to define $\hat C: \Gamma \times \Sigma \to \{1,2,...,K\}$ and require that
\begin{align}
\Pr[\hat C(r(\omega),\varsigma)] = \Pr[C(\omega)|r(\omega)]. \label{eq:criterion}
\end{align}
where $\varsigma$ ``generates the randomness'' in $\hat C$ needed to bring it into agreement with $C(\omega)$; see, e.g.\ approaches based on Bayesian neural networks \cite{BNN1,BNN2}.  Alternatively we could require the classifier to be a mapping from $\Gamma$ onto  $\X$, so that $\hat C$ is explicitly a probability model.  This is the motivation behind cross-entropy minimization \cite{DNN}, for example.  But while such approaches have been successful in certain applications, they neither account for situations in which $\Pr[C]$ changes, nor do they directly model $\Pr[r|C]$.  Thus, we claim that Eq.\ \eqref{eq:criterion} does not fully unravel the connection between classifiers and probability.

The solution to these problems, which we first studied in the context of a binary setting \cite{PartI}, is to consider a \textit{family} $F = \{\hat C(r;q)\}$ of deterministic classifiers indexed by a vector of affine parameters $q \in \X$, which in fact play the role of a prevalence;\footnote{Because one can always convert a stochastic classifier into a deterministic one via the mapping $\hat C(r) \to \argmax_{C} \Pr[\hat C(r) = C]$,  it is sufficient to consider deterministic $\hat C$.  See Sec.\ \ref{sec:discussion}.} see Refs.\ \cite{Langford05,Halck} for similar ideas.  In other words, we seek to promote the test prevalence (which is a fixed property) to the status of an independent variable and define 
\begin{align}
\hat C:\Gamma \times \X \to \K.  \label{eq:promote}
\end{align}
To justify this, note that one always classifies samples {\it from a test population} $\Omega(\chi)$.   Thus, it stands to reason that the set $\X$, which is traditionally omitted from the definition of a classifier, should be included because $\chi$ directly impacts test points $\omega$.  As Fig.\ \ref{fig:binex} illustrates, this allows one to adapt a classifier to a population $\Omega(\chi)$ by setting ${\hat C = \hat C(r,\chi)}$.\footnote{As a historical note, Refs.\ \cite{Langford05,Halck} did not take the formal step of defining $\hat C$ as a function that acts on $\Gamma \times \X$.  While this generalization may seem trivial, it is critical for our analysis.  By its very nature, a classifier defined via Eq.\ \eqref{eq:oldclassifier} \textit{cannot} be fully reconciled with probability theory, whereas Eq.\ \eqref{eq:promote} can.}

To extract information about the $\Pr[r|C]$, however, we require an additional property.
\begin{definition}\label{def:monotoneclassifier}
Let $\X$ be the probability simplex, and let $\hat C: \Gamma \times \X \to \K$.  For $k < k'$ and $\alpha \in \XT$, let  $q(\alpha;k,k'):\XT \to \X$ be the vector whose elements are $q(\alpha;k,k')_k=\alpha_1$, $q(\alpha;k,k')_{k'}=\alpha_2=1-\alpha_1$, and $q(\alpha;k,k')_j=0$ for $j\ne k, k'$.  Then we say that  $\hat C$ is a \textbf{monotone classifier} if for every pair $k,k'$, $k< k'$:
\begin{itemize}
\item there exists a non-empty subset $\Gamma_{k,k'}\subset \Gamma$ such that  $\hat C(r,q(\alpha;k,k'))$ is  a monotone decreasing function of  $\alpha_1$ for every $r\in \Gamma_{k,k'}$;
\item  the $\Gamma_{k,k'}$ cover $\Gamma$, i.e.\ $\displaystyle \bigcup_{k,k'}\Gamma_{k,k'} = \Gamma$.
\end{itemize}
We refer to the argument $q$ of such a monotone classifier $\hat C(r,q)$ as an \textbf{affine prevalence}. %We also refer to $q$ as the \textbf{affine prevalence}.
\end{definition} 

\begin{remark}
We refer to both Eq.\ \eqref{eq:oldclassifier} and \eqref{eq:promote} as classifiers.  Throughout, we distinguish between these two definitions by explicitly specifying their domains of action.  
\end{remark}

While definition \ref{def:monotoneclassifier} seems abstract, its interpretation is straightforward in the context of a binary classifier.  In this case, $\Gamma_{k,k'}=\Gamma$, and $\hat C$ reduces to a function of the form $\hat C(r,(\alpha_1,1-\alpha_1))$ for $\alpha_1 \in [0,1]$.  Figure  \ref{fig:binex} shows a simple monotone classifier sampled on a discrete grid of $\alpha_1$ values.  In fact, these are Bayes-optimal classifiers associated with populations $\Omega(\chi)$ having test prevalence values $\chi = q =(\alpha_1,1-\alpha_1)$ [see Sec.\ \ref{sec:examps} for details]. This example illustrates an intuitive and powerful property: monotone classifiers have a set inclusion (or monotone class \cite{LiebLoss}) structure parameterized by $\alpha$, and as a result, the assigned class of each point changes at one and only one value of $\alpha$.  Key tasks in the next sections are to (i) demonstrate that  Bayes classifiers satisfy Def.\ \ref{def:monotoneclassifier}, (ii) prove that the value of $\alpha$ at which the class changes determines $\Pr[r|C]$; and (iii) show that the sets $\Gamma_{k,k'}$ are tied to the supports of $\Pr[r|C]$.  But to arrive at this point, we must formalize the monotone structure in terms of set theory; see Sec.\ \ref{sec:level-set}.  For now it is sufficient to note that in the context of Fig.\ \ref{fig:binex}, we may view  $q$, as a parameter that ``deforms one classifier into another.''  It is for this reason that we refer to this variable as an affine prevalence, i.e.\ to highlight its connection to geometry.\footnote{We believe that the connection between our work and differential geometry  runs much deeper than initially meets the eye.  The terminology ``affine prevalence'' is also meant to suggest the idea that the density ratio level-sets discussed in the next sections may have a viable interpretation as tangent spaces  connected by $q$.  See Sec.\ \ref{sec:discussion}, as well as the curves on Fig.\ \ref{fig:binex}.}   {\it Note  that $q\in \X$ should be treated as a variable, not a fixed parameter.}

\begin{figure}
\begin{center}
\includegraphics[width=8cm]{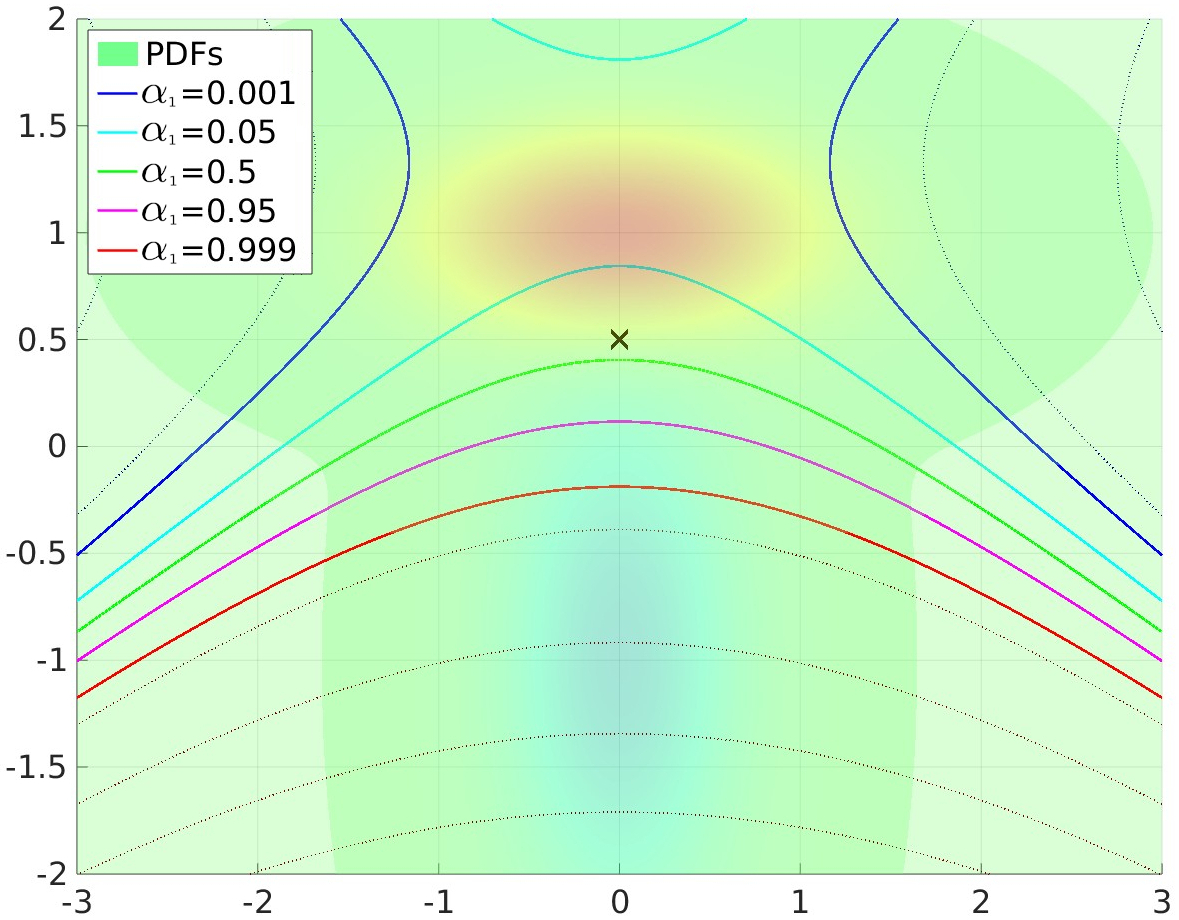}
\end{center}
\caption{Example of a binary monotone classifier.  The color scale denotes the probability densities of two Gaussian distributions quantifying the probabilities $\Pr[r|C]$ for two classes.  The top (red) density is associated with class $1$, and the bottom (blue) density is associated with class $2$.  The hyperbolas are classification boundaries for different values of the affine prevalence $q=(\alpha_1,1-\alpha_1)$.    Each boundary corresponds to a different classifier.  For the plotted values of $\alpha_1 \ge 0.05$, the points on the same side as the foci are assigned to class $2$, and those on the other side are assigned to class $1$.  For $\alpha_1 = 10^{-3}$, the assignment is reversed.  These boundaries are also Bayes-optimal for test populations $\Omega(\chi)$ having test prevalence values $\chi=q$.  [Shading does not correspond to any particular test prevalence.]  Note that as a function of $\alpha_1$, the point $x$ changes class one and only one time.  In fact, $\hat C(r,q)$ is a monotone decreasing function of $\alpha_1$.  The classification boundaries also suggest an important set inclusion structure associated with the classification domains.  See the top-right panel of Fig.\ \ref{fig:ternary} (which shows the same distributions) and Eqs.\ \eqref{eq:gausspdfs}--\eqref{eq:means2} for more details.}\label{fig:binex}\end{figure}

To summarize our perspective then, we answer the question, ``what is a classifier,'' as follows.  To be consistent with probability, an arbitrary classifier must be defined by Eq.\ \eqref{eq:promote}, and when equipped with the monotonicity property, can be made equivalent to a set of (possibly implicit) parameterized probability models for $\Pr[r|C]$.  It follows that training such a classifier is equivalent to statistical regression on those models.  We emphasize, however, that monotonicity is a necessary, but not always sufficient condition for this correspondence to hold.  In the following sections, we derive additional requirements, the classifier versions of the law of total probability and normalization.

%In this context, \textbf{we advise the reader that our perspective is not inherently Bayesian.}  It would appear that the test prevalence $\chi$ is not informed by the training process, in which case it would have to be specified as a prior.  However, we have previously shown that for a test population $\Omega(\chi)$, $\chi$ can be determined objectively in terms of unbiased estimators that converge with the amount of test data \cite{Luke_Multi,Patrone22_2}; herein we present additional results to this effect.  Even when only a single test point is available for classification, our analysis can reduce to both Bayesian approaches and those that assume $\chi$ differs between training and test populations.  Thus, we claim that our analysis generalizes and subsumes these perspectives.  

\subsection{Implications for UQ}
\label{subsec:imp4UQ}

Momentarily assuming that the claims of Sec.\ \ref{subsec:overview} are true, we pause to consider how our perspective informs the broader task of UQ for classification and ML.  The reasons for doing so are twofold.  First, we primarily consider classification in the idealized limit of infinite data and perfect optimization.  Framing this analysis in the context of UQ helps to distinguish sources of uncertainty and thereby clarify the assumptions used herein.  Second, this idealized limit yields the smallest-possible uncertainty in the class labels, upon which all other effects build.  Thus, our discussion highlights the fact that the relationship between classifiers and probability distributions is fundamental to UQ of ML.   We refer the reader to Refs.\ \cite{UQ_Survey1,SandiaUQML,UQ_Survey2,UQ_Survey3} and the references therein for other perspectives on this topic.

If we therefore consider a best-case scenario, e.g. $\hat C$ is chosen to be Bayes-optimal, it is still possible that $0<\Pr[C|r]<1$, meaning that $r$ simply does not provide unambiguous information about how to classify data.   This motivates the following definition.
\begin{definition}\label{def:inherent}
Let 
\begin{align}
u(r) &= 1- \max_C \left\{\Pr[C|r]\right \} = 1 - \max_C \left\{\frac{\Pr[r|C]\Pr[C]}{\Pr[r]}\right \}. \label{eq:pointwiseinherent}
\end{align}
We refer to $u(r)$ as the \textbf{pointwise irreducible (or inherent) uncertainty.}
\end{definition}
In UQ for ML literature, irreducible uncertainty is sometimes referred to as entropy \cite{UQ_Survey1} or variability in real-world effects  and measurement noise \cite{UQ_Survey2}.  Irrespective of the name, the key observation is that given the choice of input space $\Gamma$, our confidence in assigning a class to  a point $r$ is bounded by $u(r)$.  Thus, Eq.\ \eqref{eq:pointwiseinherent} is also the pointwise version of the Bayes error \cite{Inherent}. 

From the perspective of uncertainty \textit{propagation}, it is notable that $u(r)$ depends on a product of the prevalence $\Pr[C]$ and conditional probability $\Pr[r|C]$.  This has several implications.  For example, $\Pr[C]$ and $\Pr[r|C]$ are separable, and hence these terms and their uncertainties can be studied in isolation.  Moreover, we generally expect uncertainties in $\Pr[C]$ and $\Pr[r|C]$ to be independent, since one is a property of a population and the other of data.  Thus, an empirical realization of $u(r)$ can be expressed as
\begin{align}
\tilde u(r) = 1 - \max_C \left\{\frac{\left[\Pr[C] + \epsilon_\chi(C) \right]\left[ \Pr[r|C] + \epsilon_{\rm dist}(C) \right]}{\sum_{C'}\left[\Pr[C'] + \epsilon_\chi(C) \right]\left[ \Pr[r|C'] + \epsilon_{\rm dist}(C) \right] } \right\} + \epsilon_{\rm num}, \label{eq:emp_inherent}
\end{align}
where $\epsilon_\chi(C)$ quantifies uncertainty in the prevalence estimate as a function of the class $C$, $\epsilon_{\rm dist}(C)$ quantifies uncertainty associated with the conditional probability distributions, and $\epsilon_{\rm num}$ quantifies other sources of uncertainty that may be associated with numerical realizations of $u(r)$.  

To understand how sources of uncertainty contribute to $\epsilon_\chi$ and $\epsilon_{\rm dist}$, consider the typical steps in a classification process.  Often one begins by collecting training and test data.  Several authors have demonstrated that the test prevalence $\chi=\Pr[C]$ can be objectively estimated without recourse to classification, which is the so-called \textit{quantitation or quantification problem} \cite{QuantificationBook,Quantification71,Quantification2005,Quantification2009,Quantification2016,Quantification2017,Quantification2018,Quantification2022,Quantification2024}.  Moreover, it is well known that the amount of training and test data controls the uncertainty $\epsilon_\chi$ in such estimates, so that we refer to this effect as  \textbf{sampling uncertainty}.  The next step is often to choose a \textbf{hypothesis set} $\H$ of functions that in principle include the classifier of interest.  By this, we mean the hypothesis set is assumed to contain a $\hat C^\star$ that minimizes a chosen loss function $\L$.  In this work, we consider the $0$--$1$ loss, so that $\hat C^\star$ is the Bayes optimal classifier.  However, given finite training data, one must approximate $\L$ in terms of an empirical loss function $\hat \L$ whose minimizer may differ from $\hat C^\star$, even when $\hat C^\star \in \H$.  This is another manifestation of sampling uncertainty, and this is equivalent to errors induced in fitting distributions for $\Pr[r|C]$ to finite data.  Such effects contribute to $\epsilon_{\rm dist}$.   

In addition to sampling uncertainty, the process of constructing a classifier typically involves optimization of $\hat \L$ over the set $\H$.  It is well-known that empirical loss functions have complicated and non-convex structures.   Thus, a common problem in more advanced classification algorithms (e.g.\ neural networks) is to ensure convergence of the training process.  We refer to this effect as \textbf{optimization error},\footnote{We use the term ``error'' when we anticipate there is an underlying deterministic truth, whereas ``uncertainty'' is used more broadly to encompass phenomena for which stochasticity plays some role in our lack of knowledge.  Thus, optimization error is distinct from sampling uncertainty since the former could, in principle, be fixed by ensuring convergence of the training process, whereas the sampling is always subject to luck-of-the-draw.} and this effect contributes to $\epsilon_{\rm dist}$; see Refs.\ \cite{Ward1,Ward2} for analysis of such problems.  If in addition to this, the hypothesis set $\H$ does not contain $\hat C^\star$, then we may incur additional \textbf{model-form error} \cite{SmithUQ,BoisvertUQ}, which is equivalent to fitting the training data to a family of distributions that does not include $\Pr[r|C]$.  This also contributes to $\epsilon_{\rm dist}$.  In the analysis that follows, we also encounter a concept of \textbf{binning uncertainty}, which is tied to fact that, given finite data, we can only quantify inherent uncertainty to within certain intervals controlled by how densely one samples level-sets.  This effect contributes to $\epsilon_{\rm num}$ and is illustrated in Fig.\ \ref{fig:binex}, where we only sample a finite number of boundary classifiers; see also Ref.\ \cite{Langford05}.  

\begin{figure}
\begin{center}\includegraphics[width=10cm]{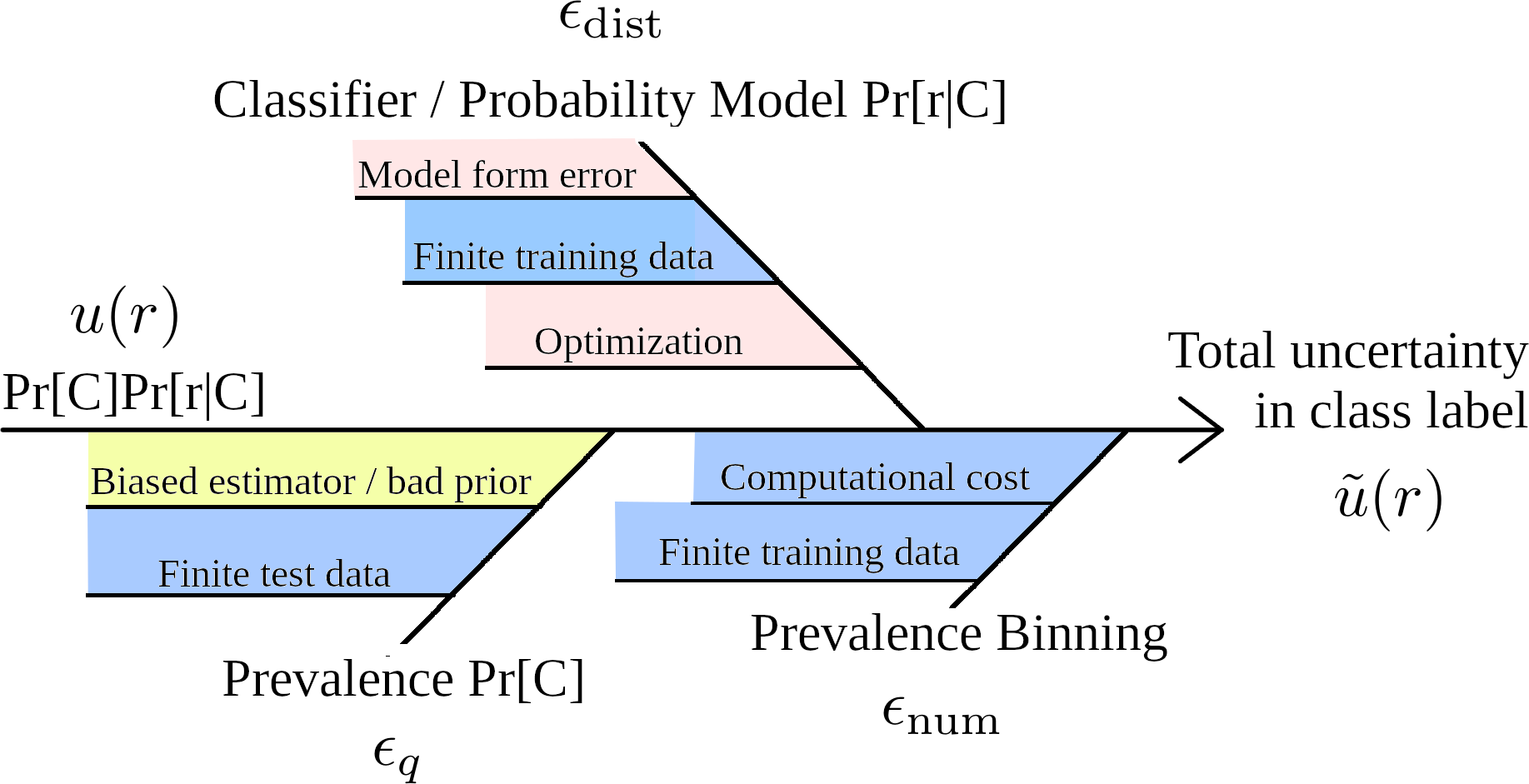}\end{center}\caption{Fishbone diagram characterizing sources of uncertainty and how they propagate into the total class uncertainty estimate.  The total uncertainty $\tilde u(r)$ in the class label is a given by the propagation of all uncertainties through $\Pr[C]\Pr[r|C]$.  Uncertainty in each of these terms is indicated by the spines; see Eq.\ \eqref{eq:emp_inherent} and the surrounding text for definitions of each symbol.  Sampling effects that can be controlled by collecting more data are shaded in blue, whereas effects shaded in red may be difficult to control and depend on in-depth knowledge of the underlying true probability models or convergence of the training routine.  The effect shaded in yellow, i.e.\ biased prevalence estimator or bad prior, may or may not be controllable, depending on the amount of test data available.  However, in the limit of large amounts of test data, it is generally possible to use unbiased estimators for the test-population prevalence.  }\label{fig:fishbone}
\end{figure}

Figure \ref{fig:fishbone} shows how these sources of uncertainty are connected to one another and propagate through a \textit{fishbone diagram}, which yields the total uncertainty in the class labels.  Note in particular that construction of the conditional distributions $\Pr[C|r]$ is equivalent to estimating the inherent uncertainty, assuming nothing else contributes.  This figure, as well as Eq.\ \eqref{eq:emp_inherent}, also clarifies our key perspective.  If we cannot unambiguously associate a classifier $\hat C$ with the conditional distributions $\Pr[r|C]$, it is not possible to even compute $\tilde u(r)$.  Moreover, this association cannot be made in an \textit{ad hoc} or \textit{post-hoc} manner.  The process of training a classifier itself can induce probabilistic information in $\hat C$, and auxiliary modeling choices on top of this can lead to unquantifiable model-form errors or worse, inconsistency with probability.  We also observe that because $\epsilon_\chi$ depends primarily on sampling, it is in generally a source of epistemic (i.e.\ reducible) uncertainty.  Thus, $\Pr[r|C]$ fully determines the aleatoric or irreducible uncertainty, and our main task amounts to determining when this concept can be meaningfully applied to a classifier.  For this reason, the remainder of the manuscript primarily focuses on settings in which $u(r)$ is the only source of uncertainty.

\section{Fundamental Mathematical Concepts}
\label{sec:context}

Following on the perspective outlined above, an analysis of ML can be split along two parallel tracks: independent constructions of $\Pr[C]$ and $\Pr[r|C]$.  The former problem has been rigorously studied in a variety of contexts spanning diagnostics and epidemiology to machine learning, so that we do not consider it here.  We refer readers to works on prevalence estimation \cite{GladenRogan,Patrone22_2,OldPrevOpt} and the quantitation problem \cite{QuantificationBook,Quantification71,Quantification2005,Quantification2009,Quantification2016,Quantification2017,Quantification2018,Quantification2022,Quantification2024}, for example.  Here we focus on establishing an equivalence between classifiers $\hat C$ and distributions $\Pr[r|C]$ in the binary context.  We remind the reader that while many of the definitions in this section have standard counterparts in the literature, we modify them so as to extract additional structure needed for later analysis.  We indicate differences where appropriate.  

Before proceeding, it is also useful to clarify that our main task can be further divided into two questions: (i) given a Bayes-optimal classifier, can we extract information about the true underlying $\Pr[r|C]$; and (ii) given an arbitrary classifier $\hat C$, is there a natural way to associate with it distributions $\Pr[r|C]$?  This section focuses the first question in a binary setting; Section \ref{sec:level-set} addresses both questions in a multiclass setting.

\subsection{Classifiers Revisited: A Set Theoretic Perspective}
\label{subsec:coredefs}

We begin by characterizing the ``true'' properties of the test population, although these are rarely directly observable.  In particular, let $r(\omega)$ and $C(\omega)$ be the measurement and true class of a sample point $\omega$ as discussed previously.  We use the notation $P_j(r)$ to denote the probability density function of the value $r(\omega)$ conditioned on $C(\omega)=j$.  Herein {\bf we always consider {\it {\bf bounded}} densities} for which a given point $r$ has zero measure, so that we may assume absolute continuity of measure \cite{Tao}.  Given $\Omega(\chi)$, the law of total probability yields the probability density of a measurement outcome $r$ for the test population as 
\begin{align}
Q(r;\chi)=\sum_j \chi_jP_j(r).
\end{align}

As suggested by Fig.\ \ref{fig:binex}, a guiding principle of our analysis is the idea that classifiers are isomorphic to partitions of $\Gamma$, since $\hat C(r,q)$ maps $\Gamma \times \X$ to a discrete set.  This motivates to the following definition.
\begin{definition}\label{def:partition_induced}
Let $\hat C:\Gamma  \to \K$ be a classifier.  Let $U=\{D_j\}$ be a collection of sets for $j\in \K$ constructed so that
\begin{align}
D_j = \{r:\hat C(r)=j\}.
\end{align}
Then we say that $U$ is the \textbf{partition induced by the classifier} $\hat C$.  Moreover, let $U=\{D_1,...,D_K\}$ be a partition of $\Gamma$ with $K$ elements.  Then we say that the classifier defined by
\begin{align}
\hat C(r) = j \mathbb I\left[r\in D_j \right]
\end{align}
is \textbf{the classifier induced by the partition} $U$.
\end{definition}

\begin{remark}
Definition \ref{def:partition_induced} leverages the concept of a classifier given by Eq.\ \eqref{eq:oldclassifier}.  This is to emphasize that the relationship between partitions and classifiers does not depend on our reinterpretation of the latter via Eq.\ \eqref{eq:promote}, but is instead a general, albeit simple, isomorphism.   For this reason, we refer to partitions $U$ of $\Gamma$ as classifiers.  Note also that we may trivially extend these definitions to settings in which $\hat C:\Gamma \times \X \to \K$ by simply promoting the partition $U \to U(q)$, i.e.\ to a mapping between $\X$ and the set of partitions on $\Gamma$.  We denote the isomorphism between partitions and classifiers as $U \cong \hat C$.  %We also use the notation $\hat C(r,q;U)$ to denote the classifier induced by the partition $U$, whereas $U(q;\hat C)$ is the partition induced by the classifier $\hat C$.  
\end{remark}

The identification of a classifier with a partition $U$ is useful because sets are natural objects on which to define probability measures.  As such, this isomorphism permits us to carry information about the $P_j(r)$ onto $\hat C$.  To foreshadow such issues, we consider two  implications of the set-theoretic perspective. 

First, it allows us to recast Bayes classifiers as follows \cite{RW,Patrone21_1}.
\begin{definition}\label{def:bayesoptimal}
Let the $P_j(r)$ be the class-conditional probability densities associated with a two-class setting (i.e.\ $j\in \{1,2\}$), and let $q=(q_1,q_2)\in \XT$ be an affine prevalence.  Then the Bayes optimal classifier for a test population $\Omega(\chi)$ having test prevalence $\chi=q$ is given by the partition $U^\star=\{D^\star_1(q),D^\star_2(q)\}$, where
\begin{subequations}
\begin{align}
D_1^\star(q)&=\{r:q_1 P_1(r) > q_2 P_2(r)\} \cup B^\star_1(q), \label{eq:d1} \\ 
D_2^\star(q)&=\{r:q_1 P_1(r) < q_2 P_2(r)\} \cup B^\star_2(q),  \label{eq:d2}
\end{align}
and where $B^\star_1(q)$ and $B^\star_2(q)$ form an arbitrary partition of 
\begin{align}
B^\star(q) = \{r:q_1 P_1(r) = q_2 P_2(r) \}.  \label{eq:bset}
\end{align}
\end{subequations}
\end{definition}

\begin{remark}
Definition \ref{def:bayesoptimal} is pointwise optimal, i.e.\ it minimizes Eq.\ \eqref{eq:pointwiseinherent}.  The analysis herein always assumes this stronger form of the Bayes optimal classifier, which stands in contrast to a weaker version that only minimizes the expected error.  This weaker definition of Bayes optimality can differ from Def.\ \ref{def:bayesoptimal} on sets of measure zero and importantly, it may not minimize Eq.\ \eqref{eq:pointwiseinherent} for all values of $r\in \Gamma$.  That one realizes a strong Bayes classifier in practice is an assumption.  Classifiers are often trained by minimizing an expected error or surrogate loss function \cite{MLBook,Zhang04,Hconsistency1,Metrics1,Metrics2,Metrics3,Metrics4,Metrics5,Metrics6}, which is not guaranteed to yield Eqs.\ \eqref{eq:d1}--\eqref{eq:bset}.  It is also important to note that even Def.\ \ref{def:bayesoptimal} only defines a classifier up to an equivalence class that can differ in the choice of $B^\star(q)$.  \textbf{We alert the reader that we always adopt this more general definition of a monotone classifier $\hat C:\Gamma \times \X \to \K$ as an equivalence class of partitions, where members of the equivalence class can differ in their boundary sets, i.e. at the affine prevalence values where the class changes.}
\end{remark}

Definition \ref{def:partition_induced} also suggests a re-interpretation of the monotonicity property of $\hat C(r,q)$; see also Fig.\ \ref{fig:binex}. [See Chapter 1 of Ref.\ \cite{LiebLoss} for background on monotone classes.]  
\begin{lemma}[Monotone Class Property of Monotone Classifiers]\label{lem:setinclusion}
Assume that the function ${\hat C:\Gamma \times \XT \to \{1,2\}}$ is a binary classifier, and let $\hat C(r,q) \cong U(q)$.  Then $\hat C$ is a monotone classifier of the form $\hat C(r,q)=\hat C(r,(q_1,1-q_1))$ if and only if for each $j\in \{1,2\}$, the $D_j(q)\in U(q)$ for all $q\in[0,1]$ form a monotone class having the property 
\begin{subequations}
\begin{align}
 \qquad & D_1(q) \subset D_1(q'), & \forall \quad q,q',\quad q_1 < q_1', \\
 \qquad & D_2(q) \supset D_2(q'), & \forall \quad q,q',\quad q_1 < q_1'.
\end{align}  
\end{subequations}
\end{lemma}
Proof of Lemma \ref{lem:setinclusion} is tedious but straightforward; it relies on monotonicity and the fact that the $D_j(q)$ are an equivalence class.  The key steps leverage the facts that: (i) once a point $r$ switches class as a function of $\alpha$, it cannot switch back; and (ii) for a fixed $r$, we are free to choose the class label at the affine prevalence at which $\hat C$ is discontinuous.  The following result is a trivial consequence of Lemma \ref{lem:setinclusion}.
\begin{corollary}
A (strong) Bayes optimal classifier is a monotone classifier.  
\end{corollary}

From a practical standpoint, the monotone class property is useful because it allows us to meaningfully consider certain types of limits of classifiers parameterized by $q$.  In Ref.\ \cite{Langford05}, this property appears to have been implied by the analysis, but not made explicit.   In the next section we show that this ability to take limits is critical for extracting information about the $P_j(r)$.  

As an aside, we also speculate that Lemma \ref{lem:setinclusion} may be of inherent interest.  In probability and measure theory, it is well-known that in order to construct a probability space, it is necessary to identify a $\sigma$-algebra, which is subsequently used to construct measurable sets \cite{LiebLoss}.  The monotone class theorem (see e.g.\ Chapter 1 of Ref.\ \cite{LiebLoss}) provides a prescription for constructing a $\sigma$-algebra from a monotone class.  Given that our aim is to establish a more rigorous connection between classifiers and probability theory, it is thus surprising that the Bayes classifier appears to induce such a monotone class.  A more in-depth analysis of this observation is left as an open question.    

\subsection{Extracting Information about $\Pr[r|c]$ in the Binary Case}
\label{subsec:PrC}

Bayes classifiers are determined by the $P_j(r)$, but the inequality structure of Eqs.\ \eqref{eq:d1}--\eqref{eq:bset} means that $U^\star(q)$ cannot be directly used to deduce these conditional densities.  To do so, it is useful to consider the boundary set, which  can be expressed as
\begin{align}
B^\star(q) = \{r:q_1 P_1(r)=q_2 P_2(r)\} = \left \{r: \frac{q_1}{q_2} = \frac{P_2(r)}{P_1(r)} \right \}, \label{eq:bstar2}
\end{align}
where we temporarily ignore issues associated with $P_1(r)=0$.  In Eq.\ \eqref{eq:bstar2}, the equality 
\begin{align}
 \frac{q_1}{q_2} = \frac{q_1}{1-q_1} = \frac{P_2(r)}{P_1(r)} \nonumber 
\end{align} 
is interesting for two reasons.  First, $B^\star(q)$ is a level-set of the ratio $P_2(r)/P_1(r)$, and thus it stands to reason that the latter quantity is special in some way.  Second, this equation suggests promoting the affine-prevalence from the role of a free parameter to that of a function of $r$, which yields the value of $q$ for which $r\in B^\star(q)$.  In other words, it seems reasonable to reinterpret the statement $r\in B^\star(q)$ as a ``set-theoretic equation'' for an unknown $q$, which we could express informally by writing ``$r\in B^\star(q;r)$'' for a fixed value of $r$.  The solution to this ``equation'' is given by some function $\q(r)$, which we define below.  In particular, consider the following.
\begin{definition}\label{def:densityratio}
Let $\Gamma_{j,k}$ be the union of supports of the bounded functions $P_j(r)$ and $P_k(r)$.  The ratio $R_{j,k}: \Gamma_{j,k} \to [0,\infty]$ defined as
\begin{align}
R_{j,k}(r) = \frac{P_k(r)}{P_j(r)} \label{eq:relprob}
\end{align}
is the {\bf density ratio} between classes $j$ and $k$, where we interpret the situation $R_{j,k}(r)= \infty$ to correspond to $P_j(r)=0$.
\end{definition}

\begin{remark}
Two important comments are in order.  
\begin{itemize}
\item The appearance of $\Gamma_{j,k}$ in Def.\ \ref{def:densityratio} is non-trivial because $R_{j,k}$ is indeterminate outside of this set.  In the binary case, we can simply overcome this by setting $\Gamma = \Gamma_{1,2}$.  However, in the multiclass case the situation is more complicated if the supports $\Gamma_{j,k}$ are partially disjoint as $j$ and $k$ vary.  In particular, it is not clear how to define $R_{j,k}$ on $\Gamma_{j',k'}/ \Gamma_{j,k}$ if $\Gamma_{j,k} \ne \Gamma_{j',k'}$, since the density ratio takes the indeterminate form $0/0$ on this set.  Addressing this issue is a key task in Sec.\ \ref{sec:level-set}.
\item For similar reasons, {\bf boundedness of the $P_j(r)$ is henceforth a critical assumption of our analysis.}  Because we do not assume continuity of the underlying densities, it is difficult, if not impossible, to resolve indeterminate forms such as $\infty/\infty$, which we would need to address if we did not assume boundedness. 
\end{itemize}
\end{remark}

\begin{definition}
Let $R_{j,k}(r)$ be a density ratio.  We refer to $\q_{j,k}: \Gamma_{j,k} \to [0,1]$ defined via
\begin{align}
\q_{j,k}(r) &= \frac{R_{j,k}(r)}{1+R_{j,k}(r)} = \frac{P_k(r)}{P_k(r) + P_j(r)} \label{eq:prevfun}
\end{align}
as the {\bf prevalence function}.
\label{def:prevfun}
\end{definition}

\begin{remark}
We adopt the convention that $R_{j,j}(r)=1$ and $\q_{j,j}(r)=1/2$, even when $P_j(r)=0$.  These equalities are important in multiclass settings.  
\end{remark}

The motivation for studying $\q_{j,k}(r)$ and $R_{j,k}(r)$ arises from three straightforward observations.  First, in the binary case it is clear that given $\q_{1,2}(r)$, this function can be inverted to yield the density ratio $R_{1,2}(r)$.  Second, the density ratio (and thus $\q_{1,2}$) is sufficient for computing the inherent uncertainty $u(r;\chi)$ of a test population.  To see this, note that in the binary case,
\begin{align}
1-u(r;\chi) = Z(r;\chi,U^\star) &= \frac{\chi_1 P_1(r) \mathbb I[r\in D_1^\star(\chi)] + \chi_2P_2(r) \mathbb I[r\in D_2^\star(\chi)]}{\chi_1 P_1(r)  + \chi_2P_2(r) } \nonumber \\
&= \frac{\chi_1 \mathbb I[r\in D_1^\star(\chi)] + \chi_2 R_{1,2}(r) \mathbb I[r\in D_2^\star(\chi)]}{\chi_1  + \chi_2 R_{1,2}(r) }.
\end{align}
In other words, the densities $P_j(r)$ are not actually needed.  Third, and perhaps most importantly, $\q(r)$ solves $r\in B^\star(q)$ in the sense that $\q_{1,2}(r)=q_1 \iff r \in B^\star(q)$.  Thus, if we can find the value of $q$ for which $r\in B^\star(q)$, we know $\q_{1,2}(r)$, and hence the $R_{1,2}(r)$.  

Intuitively, we can solve this set-theoretic equation by finding the value of the affine prevalence $q$ at which the class switches.  To make this precise, however, we need to guarantee that such a solution always exists and establish certain other properties.  It turns out that this is equivalent to demonstrating that the $B^\star(q)$ are level-sets in the following sense \cite{PartI}.  

\begin{proposition}[Level-Set Representation of a Binary Classifier]
Assume that the function $\hat C^\star: \Gamma \times \XT \to \{1,2\}$ is a binary Bayes classifier, and let $\hat C^\star(r,q)\cong U^\star(q)$.  Then the sets $B^\star(q)$ have the following properties:
\begin{itemize}
\item[I.]  ${B^\star(q) \subset D_1^\star(q')/B^\star(q')}$ and ${B^\star(q') \subset D_2^\star(q)/B^\star(q)}$ for $q'_1 > q_1$, and in particular, $B^\star(q) \cap B^\star(q') = \emptyset$ for $q\ne q'$;
\item[II.] for every $r\in \Gamma$, there exists one and only one value of $q$ for which $r \in B^\star(q)$, and in particular, $r  \in  B^\star(q) \iff \q_{1,2}(r)=q_1$; that is, $r\in B^\star(\q(r))$ for $\q(r)=(\q_{1,2}(r),1-\q_{1,2}(r))$;
\item[III.] $\chi_1 = \q_{1,2}(r)\in (0,1)$ yields the test prevalence $\chi=(\chi_1,1-\chi_1)$ for which $r$ has 50\% probability of belonging to either class; i.e.\ $B^\star(\chi)$ is the 50\% probability level-set.
\end{itemize}\label{prop:ls}
\end{proposition}

Proposition \ref{prop:ls} formalizes what it means for the boundary sets $B^\star(q)$ to be level-sets of the density ratios: they must be mutually exclusive, cover $\Gamma$, and be ordered in a meaningful way in terms of the $P_j(r)$.  This is also a geometric way of stating that the equation $r\in B^\star(q)$ always has a solution $\q_{1,2}(r)$.   From this perspective, the affine prevalence $q$ can be interpreted as the parameter that orders the level sets.

In light of monotonicity and Prop.\ \ref{prop:ls}, it is straightforward to prove that class-switching can be used to determine $\q_{1,2}(r)$.  In particular:    
\begin{proposition}[Class-Switching Representation of the Density Ratios]
Assume that $\hat C^\star: \Gamma \times \XT \to \{1,2\}$ is a binary Bayes classifier.  Then for every $r$, either: $\q_{1,2}(r)\in \{0,1\}$; or $\q_{1,2}(r) = q_l(r)=q_h(r)$, where
\begin{align}
q_l(r)=\sup \{q_1: \hat C^\star(r,q) = 2\}, && q_h(r)= \inf \{q_1: \hat C^\star(r,q) = 1\}.  \label{eq:class_switch}
\end{align}
\label{prop:bfr}
\end{proposition}

\begin{proof}
Let $U^\star(q) \cong \hat C^\star(r,q)$.  Fix a value of $r$.  Assume first that $P_j(r) > 0$ for $j=1,2$.    By Prop.\ \ref{prop:ls}, there exists a $\tilde q$ for which $\tilde q_1\in (0,1)$ such that $r\in B^\star(\tilde q)$.  By Lemma \ref{lem:setinclusion} and Def.\ \ref{def:bayesoptimal}, we know that for $q_1 > \tilde q_1$, the point $r\in D^\star_1(q)$, and hence $\hat C^\star(r,q)=1$.  For $q_1 < \tilde q_1$, one finds $\hat C^\star(r,q)=2$.  Since $\hat C^\star(r,q)$ is binary and monotone in $q$ for fixed $r$, there can only be one point of discontinuity, which must be $\tilde q$.  Equality of the supremum and infimum also follows by monotonicity.  

If $P_1(r) = 0$ and $P_2(r) > 0$, by the Def.\ \ref{def:bayesoptimal}, then $r\in D_2^\star(q)$ for any finite value of $q_1$, and in particular, one finds that $r\in B^\star(q)$ for $q_1=1$.  Thus we set $\q_{1,2}=1$.  A similar argument yields $\q_{1,2}=0$ when $P_1(r)>0$ and $P_2(r) = 0$.  
\end{proof}

\begin{remark}
Reference \cite{Langford05} considered a variation on Prop.\ \ref{prop:bfr}, which they called \textit{probing.}  A key distinction of our analysis is its ability to extend to the multiclass case when the conditional distributions have disjoint supports.  
\end{remark}

\section{Class-Switching in the Multiclass Setting}
\label{sec:level-set}

The class-switching representation of the density ratios demonstrates how to extract probabilistic information from binary Bayes classifiers.  While the $R_{j,k}(r)$ do not fully specify all of the distributional information in the $P_j(r)$, in the binary case, they are sufficient constructing $u(r)$.  

In this context, the purposes of the present section are twofold.  First, we wish to generalize the results of the previous section to a multiclass setting.  In particular, we show that given a multiclass Bayes classifier in the sense of Eq.\ \eqref{eq:promote}, we can still extract sufficient probabilistic information to compute aleatoric uncertainty.  %We accomplish this in a somewhat surprising way by constructing $\hat C^\star$ on interior of $\X$ from its properties (i.e.\ density ratios) on the boundary $\partial \X$ of the probability simplex.  To verify that such a $\hat C^\star$ is Bayes-optimal, we then directly check Eq.\ \eqref{eq:pointwiseinherent} to ensure that the probabilistic information has been computed correctly.  
A surprising outcome of this exercise is the identification of two additional properties, self-consistency and normalizability, that all Bayes classifiers satisfy.  This leads to the second purpose of the present section, which is to study the extent to which these properties, along with monotonicity, are sufficient for assigning a meaningful probabilistic interpretation to arbitrary classifiers; see also the discussion at the beginning of Sec.\ \ref{sec:context}.

\subsection{Constructing Bayes Classifiers from Pairwise Counterparts}

In a multiclass setting, extracting information about the densities $P_j(r)$ can be achieved by pairwise application of the class-switching representation.  To motivate this, observe that the multiclass Bayes classifier is given by 
\begin{align}
D_j^\star(q) &= \{r:q_jP_j(r) > q_kP_k(r), \forall k, k\ne j\} \cup B_j^\star(q) \nonumber \\
&=\{r:q_j/q_k > R_{j,k}(r), \forall k, k\ne j\}\cup B_j^\star(q), \label{eq:multibayes}
\end{align}
where $B_j^\star(q)$ is a subset of the set $\{r:\exists k, k\ne j: q_jP_j(r) = q_kP_k(r) \}$, which is the multiclass generalization of the boundary set.  As in the binary case,  the optimal classification domains only depend on the pairwise density ratios.

These observations suggest that we consider the restriction of $\hat C^\star$ to the boundary $\partial \X$ of the probability simplex as follows.

%In practice, if one has access to a Bayes classifier $\hat C^\star: \Gamma \times \X \to \K$, then one can also compute $\hat C^\star$ on its restriction to the set $\Gamma \times \partial \X_{j,k}$, where $\partial \X_{j,k}$ is the part of the boundary of $X$ on which only $q_j$ and $q_k$ are non-zero.  This means that Prop.\ \ref{prop:bfr} can be used to determine the density ratios $R_{j,k}$, from which Eq.\ \eqref{eq:multibayes} can be constructed.  We now verify this claim and demonstrate that doing so yields the inherent uncertainty for the multiclass Bayes classifier.  
%
%
%In practice, it is generally easier to construct a pairwise classifier than its multiclass counterpart.  Thus it is reasonable to assume that one always has access to the pairwise Bayes classifier $\hat C^\star_{j,k}$, from which Prop.\ \ref{prop:bfr} yields the relative density ratios $R_{j,k}$.  We now show that this information is sufficient for both constructing the multiclass Bayes classifier and estimating the inherent uncertainty associated with any population.
%
%To make use of this observation, it is helpful to ``project'' a multiclass classifier onto its pairwise counterparts via the following definition.
\begin{definition}[Pairwise Classifiers]\label{def:pairwise}
Let $\hat C:\Gamma \times \X \to \K$ be a monotone classifier for $K > 2$.  Moreover, let $q(\alpha;j,k): \XT \to \X$ for $j\ne k$ be a vector valued function of $\alpha\in \XT$ having the property that $q(\alpha;j,k)_j = \alpha_1$, $q(\alpha;j,k)_k = \alpha_2= 1-\alpha_1$, and $q(\alpha;j,k)_m = 0$ for $m\ne j,k$.  We define $\hat C_{j,k}(r,\alpha) = \hat C(r;q(\alpha;j,k))$ to be the \textbf{pairwise classifiers} induced by $\hat C$, where $\hat C_{j,k}:\Gamma_{j,k} \times \XT \to \{j,k\}$.  
\end{definition}

\begin{remark}
In contrast with Def.\ \ref{def:monotoneclassifier}, the definition of a pairwise classifier does not require that $j< k$.  However, note that there are only $K(K-1)/2$ unique pairwise classifiers for $\hat C:\Gamma \times \X \to \K$, and similarly for the $R_{j,k}$ and $\q_{j,k}$, and we can construct all such functions by only considering those for which $j< k$.  
\end{remark}

In practice, it is generally easier to construct a pairwise classifier than its multiclass counterpart.  Thus it is reasonable to assume that one always has access to the $\hat C^\star_{j,k}$, from which Prop.\ \ref{prop:bfr} yields the relative density ratios $R_{j,k}$.  We now show that this information is sufficient for both constructing the multiclass Bayes classifier and estimating the inherent uncertainty associated with any population.

\begin{construction}[Pairwise Construction of Multiclass Bayes Classifier]\label{constr:multi}
We construct a Bayes classifier $\hat C^\star: \Gamma \times \X \to \K$ in the multiclass setting using only pairwise training.  Assume that the support $\Gamma_j$ of each $P_j(r)$ is the full measurement space, i.e.\ $\Gamma=\Gamma_j$.  Let $\{\hat C_{j,k}^\star\}$ be the $K(K-1)/2$ unique pairwise Bayes classifiers.  By Proposition \ref{prop:bfr}, for any point $r$, we can determine $\q_{j,k}(r)$ in terms of the value of $\alpha$ at which the $\hat C_{j,k}^\star(r,\alpha)$ is discontinuous.

Consider next a test population $\Omega(\chi)$ with test prevalence $\chi$ \textcolor{black}{in the interior of $\X$}.  Given the pairwise prevalence functions, we can construct the pointwise accuracy of an arbitrary classifier $\hat C:\Gamma \to \K$ as 
\begin{align}
 Z(r;\chi,\hat C)&=\sum_{j=1}^K \frac{\chi_jP_j(r)}{Q(r;\chi)} \mathbb I\left(\hat C(r) = j \right) \nonumber \\
&= \sum_{j=1}^K \frac{\chi_j}{\sum_{k}\chi_k R_{j,k}(r)} \mathbb I\left(\hat C(r) = j \right) \nonumber \\
&= \sum_{j=1}^K \frac{\chi_j}{\sum_{k}\chi_k \q_{j,k}(r)/[1-\q_{j,k}(r)]} \mathbb I\left(\hat C(r) = j \right). \label{eq:MultiZacc}
\end{align}  
In order to ensure that $\hat C(r)$ is the Bayes classifier, define
\begin{align}
\hat C(r) = \argmax_j \left[\frac{\chi_jP_j(r)}{Q(r;\chi)} \right] =
 \argmax_j \left[ \frac{\chi_j}{\sum_{k}\chi_k \q_{j,k}(r)/[1-\q_{j,k}(r)]} \right], \label{eq:argmaxclass}
\end{align}
\textcolor{black}{where we may assign $\hat C(r)$ to any class attaining the maximum in the event that more than one $j$ solves Eq.\ \eqref{eq:argmaxclass}.}  Moreover, observe that the resulting $\hat Z(r;\chi,\hat C)=1-u(r)$ yields the inherent uncertainty, and thus $\hat C$ is  the Bayes classifier for test prevalence $\chi$.  That is, $\hat C(r)=\hat C^\star(r,\chi)$.   \hfill {\small Q.E.F.}
\end{construction}

A key, and perhaps surprising, outcome of Construction \ref{constr:multi} is that when the supports of the $P_j(r)$ fully overlap, a Bayes classifier never need be trained on all classes simultaneously, only pairs thereof.  This is sometimes referred to as an \textit{all-pairs} or a \textit{one-versus-one} training paradigm \cite{MultiStrategies,OneVOne1,OneVOne2,OneVOne3,OneVOne4}; see also Ref.\ \cite{Reduction} and the concept of reduction.  This approach also stands in contrast to other popular methods such as \textit{one-versus-all} \cite{OneVAll1,OneVAll2}; see also Refs.\ \cite{Pairwise1,Pairwise2}.  Construction \ref{constr:multi} was foreshadowed in Refs.\ \cite{Langford05,MultiGen} but does not appear to have been proved rigorously.  Here it is critical to note that a straightforward extension of binary class-switching to the multiclass setting relies on the rather restrictive assumption that the $\Gamma_j = \Gamma$, which is not reasonable in many settings.  

Generalizing this result to the case where $\Gamma_{j,k}\ne \Gamma$ requires significant care.  The problem arises from the fact that at a point $r$ for which $P_j(r)=P_k(r)=0$, the density ratio is indeterminate.  In principle, this means that we can just ignore such points when using Eq.\ \eqref{eq:argmaxclass}.  But in practice, one wishes to apply a classifier to test data, for which one does not know \textit{a priori} when a point falls outside the support of a given distribution.  Perhaps worse, numerical realization of a classifier is likely to lead to a situation in which $P_j(r)=P_k(r)=0$, and yet the induced pairwise classifier still yields values $\hat C^\star_{j,k}(r,\alpha)\in \{j,k\}$.  In this situation, several outcomes are possible.  For example,  it is possible that the classifier is not monotone, in which case it is not obvious how to compute $\q_{j,k}(r)$.  Alternatively, the classifier may be monotone and yield a positive, bounded value of $R_{j,k}(r)$, not an indeterminate form.  In other words, for a Bayes classifier constructed according to Eq.\ \eqref{eq:argmaxclass}, we cannot know beforehand whether a pairwise class label is spurious, and we must assume that all pairwise classifiers will be evaluated on the full set $\Gamma$, not necessarily restricted to their respective $\Gamma_{j,k}$.  Thus, we can at most say that $0\le \q_{j,k} \le 1$ for $r\notin \Gamma_{j,k}$, which is an extremely weak conclusion.

It turns out that even in such situations, we can still construct the multiclass Bayes classifier, given certain assumptions on its pairwise counterparts.  In particular, it is necessary to define reasonable extensions of the density ratios and prevalence functions to points outside of their native domains.

%: (a) use class-switching to determine whether $r$ falls within the support of a given $P_j(r)$; (b) identify when $R_{j,k}(r)$ is spurious or meaningful; and (c) thereby safely construct the multiclass Bayes classifier and inherent uncertainties.  This arises from the fact that we always consider pairs $\Gamma_{j,k}$ for which at least one class will satisfy $P_j(r)>0$, so that we can detect cases $P_k(r)=0$ via degenerate values of $\q_{j,k}(r)$.  In particular, note that by class-switching, the ratios
%\begin{subequations}
%\begin{align}
%R_{j,k}(r) &= \infty  &\iff&  &\q_{j,k}(r) = 1, \\
%R_{j,k}(r) &= 0  &\iff&  &\q_{j,k}(r)=0,
%\end{align}
%\end{subequations}
%correspond to $P_j(r)=0$, $P_k(r) > 0$ and $P_j(r)>0$, $P_k(r)=0$, respectively.  In this way, we see that by pairwise examination of the $\q_{j,k}(r)$, we can address tasks (a) and (b).  To address task (c), we require a definition of a classifier and prevalence function that extends from $\Gamma_{j,k}$ to $\Gamma$.

\begin{definition}
Let $\hat C_{j,k}$ be a pairwise classifier defined on $\Gamma_{j,k}$, i.e. $\hat C_{j,k}:\Gamma_{j,k}\to \{j,k\}$.  We say that $\hat \C_{j,k}$ is an {\bf extension of} $\hat C_{j,k}$ {\bf onto} $\Gamma$ if $\hat \C_{j,k}(r,\alpha) = \hat C_{j,k}(r,\alpha)$ for  $r\in \Gamma_{j,k}$ and $\hat \C_{j,k}(r,\alpha)\in \{j,k\}$ for $r\in \Gamma / \Gamma_{j,k}$.
\end{definition}

\begin{definition}\label{def:qextension}
Let $\hat C^\star_{j,k}$ be the pairwise Bayes classifiers induced by $\hat C^\star$, and let $\hat \C_{j,k}^\star$ be any extension of the classifiers to $\Gamma$.  We say that $\tilde \q_{j,k}$ is an extension of the prevalence function onto $\Gamma$ if $\tilde \q_{j,k}(r) = \q_{j,k}(r)$ for $r\in \Gamma_{j,k}$ and $\q_{j,k}(r) \in [0,1]$ for $r\in \Gamma / \Gamma_{j,k}$.
%\begin{align}
%\sup\{\alpha_1:\hat \C^\star_{j,k}(r,\alpha)=k\} \le \tilde \q_{j,k}(r) \le \inf\{\alpha_1:\hat \C^\star_{j,k}(r,\alpha)=j\}  \label{eq:reswitching}
%\end{align}
%for $r\in \Gamma / \Gamma_{j,k}$.  
\end{definition}

%\begin{remark}
%Note that Eq.\ \eqref{eq:reswitching} mirrors Eq.\ \eqref{eq:class_switch}, except that we do not require the supremum and infimum to be equal.  This arises from the fact that that we cannot guarantee that a realizable extension of a pairwise Bayes classifier will be monotone on the set $\Gamma / \Gamma_{j,k}$ for which it is not naturally defined.  Thus, Eq.\ \eqref{eq:reswitching} simply states that outside of the set $\Gamma_{j,k}$, the extension of the prevalence function has to be bounded by zero and one.  
%\end{remark}

\begin{construction}\label{constr:genmulti}
Let $\Gamma_j$ be the support of $P_j(r)$.  We construct the Bayes optimal classifier on $\Gamma$ under the assumption that $\Gamma_{j} \ne \Gamma$ for at least one $j$.   

To accomplish this, fix $r$ and let $\C_{m,n}^\star$ be arbitrary extensions of the pairwise classifiers onto $\Gamma$, with $\tilde \q_{m,n}$ the corresponding extensions of the prevalence function.  \textcolor{black}{As before, let $\chi$ be in the interior of $\X$.}   Clearly there must be at least one value of $m$ for which $r\in \Gamma_m$.  Assume first that $r\notin \Gamma_n$ for some $n\ne m$.   Note that $\q_{m,n}(r)$ and $\q_{n,m}(r)$ are defined on $\Gamma_{m,n}=\Gamma_m \cup \Gamma_n$, so that $\tilde \q_{m,n}(r)=\q_{m,n}(r)$ and $\tilde \q_{n,m}(r) = \q_{n,m}(r)$ on this set.  Extend Eq.\ \eqref{eq:argmaxclass} to $\Gamma$ according to
\begin{align}
\hat C(r) = \argmax_j \left[ \frac{\chi_j}{\sum_{k}\chi_k \tilde \q_{j,k}(r)/[1-\tilde \q_{j,k}(r)]} \right]. \label{eq:extargmaxclass}
\end{align}
Because $r\notin \Gamma_n$, $\q_{m,n}=0$ and $\q_{n,m}=1$.  Examining Eq.\ \eqref{eq:extargmaxclass}, the term
\begin{align}
\frac{\chi_n}{\sum_{k}\chi_k \tilde \q_{n,k}(r)/[1-\tilde \q_{n,k}(r)]}=0, 
\end{align}
irrespective of the values of $\tilde q_{n,k}(r)$ for $k\ne m$, since these are all bounded.  Thus, $\hat C(r) \ne n$.  Because $n$ was an arbitrary class for which $P_n(r)=0$, Eq.\ \eqref{eq:extargmaxclass} does not assign $r$ to any class for which $r$ is not in the support of the corresponding PDF.  

Next let $J$ be the set of classes for which $P_j(r) > 0$.  That is, $P_j(r) > 0 \implies j\in J$.  Because $\tilde q_{j,n}=0$ if $r\notin \Gamma_n$, one finds that Eq.\ \eqref{eq:extargmaxclass} reduces to
\begin{align}
\hat C(r) = \argmax_{j\in J} \left[ \frac{\chi_j}{\sum_{k\in J}\chi_k \tilde \q_{j,k}(r)/[1-\tilde \q_{j,k}(r)]} \right]. \label{eq:redextargmaxclass}
\end{align}
This is clearly the Bayes optimal classifier on the set $J$ of classes $j\in J$ for which $r\in \Gamma_j$.  Thus, Eq.\ \eqref{eq:extargmaxclass} is the  Bayes optimal classifier on $\Gamma$.  \hfill {\small Q.E.F.}
\end{construction}

\subsection{Determining When a Classifier is (not) Bayes Optimal}

Up to this point, we have analyzed the properties of Bayes classifiers, tacitly assuming that these are the outputs of a training paradigm.  In particular, we have assumed existence of the densities $P_j(r)$ and used these to study properties of $\hat C^\star$.  But we may also consider whether a set of \textit{arbitrary} pairwise classifiers implies the existence of density functions.  To accomplish this, we ask whether it is possible to know \textit{a priori} (i.e. without access to a training population) that a classifier $\hat C$ is actually Bayes optimal across all of $\Gamma \times \X$.  Moreover, we seek to deduce  properties of the conditional densities from $\hat C$.

The motivation for this question arises from the fact that classifiers are often constructed from imperfect training processes involving finite data and/or incomplete optimization.  Thus, it is useful to have techniques to assess (i) whether a training population is ``induced'' by a classifier, and (ii) if that  population is ``close'' to the available training data.  Lacking the first property, it is not clear that meaningful uncertainty estimates can be assigned to a classifier; see Sec.\ \ref{subsec:imp4UQ}.  Moreover, in applications such as image classification, it is often unclear how to express $P_j$ in an intelligible way, so that when $\hat C^\star:\Gamma \times \X \to \K$ exists, it is in fact the representation of those densities.  Equally important, it is useful to have necessary and sufficient conditions that guarantee a set of pairwise classifiers can be used to construct a multiclass Bayes classifier in the spirit of Construction \ref{constr:genmulti}.  Such conditions could inform numerical optimization algorithms in ML and address unresolved issues of probabilistic consistency discussed in Refs.\ \cite{Halck}.

As of yet, we are unaware of a complete solution to this problem, and herein we only present partial results.  In particular, it is straightforward to prove that monotone binary classifiers with ``normalizable'' density ratios are always Bayes classifiers for which we can extract density functions.  We also argue that the general problem can be restated in terms of constrained optimization.  The issue amounts to the fact that the $P_j(r)$ are required to be both normalizable and non-negative, which is non-trivial to guarantee in the multiclass case.  In the following sections, we therefore limit ourselves to (i) deriving additional properties of Bayes classifiers, (ii) showing that these are necessary conditions for an arbitrary $\hat C$ to induce densities $P_j$, and (iii) examining special cases.

\subsubsection{Law of Total Probability for Classifiers}

We begin by examining the argument of Eq.\ \eqref{eq:redextargmaxclass}, which satisfies an important normalization property.  In particular, note that  for any test prevalence $\chi$ on the \textit{interior} of $\X$ (i.e.\ $0< \chi_j < 1$ for all $j\in \K$), 
\begin{align}
\sum_{j\in J}  \frac{\chi_j}{\sum_{k\in J}\chi_k \tilde \q_{j,k}(r)/[1-\tilde \q_{j,k}(r)]} = 1,  \label{eq:oneconsistency}
\end{align}
which is verified by recasting the prevalence functions in terms of the densities $P_j(r)$.  In so doing, note that Eq.\ \eqref{eq:oneconsistency} is simply the law of total probability in disguise.

It is also trivial to show that the density ratios formally satisfy the relationship
\begin{align}
R_{j,k}R_{k,m}=R_{j,m},  \label{eq:Ridentity}
\end{align}
when they do not involve indeterminate forms.  In light of Eq.\ \eqref{eq:prevfun}, this second identity can be massaged into the form
\begin{align}
\frac{\q_{j,k}(r)}{[1-\q_{j,k}(r)]}\frac{\q_{k,m}(r)}{[1-\q_{k,m}(r)]} = \frac{\q_{j,m}}{1-\q_{j,m}}. \label{eq:qidentity}
\end{align}
Because these last two identities are clearly equivalent, we henceforth focus on Eq.\ \eqref{eq:Ridentity}, which is formally easier to manipulate.

Equations \eqref{eq:oneconsistency} -- \eqref{eq:qidentity} are significant because they are testable properties of pairwise classifiers trained in a multiclass setting.  In fact, we prove below Eq.\ \eqref{eq:Ridentity} is also equivalent to Eq.\ \eqref{eq:oneconsistency}, so that we need only test one  to verify both.  Before doing so, however, we intuitively motivate why Eq.\ \eqref{eq:Ridentity} is important.  In particular, consider a triple of density ratios in a three-class setting.  The choices
\begin{subequations}
\begin{align}
R_{1,2}(r) &= \epsilon, \label{eq:not2} \\
R_{2,3}(r) &= \epsilon, \label{eq:not1}\\
R_{1,3}(r) &= 1, \label{eq:not7}
\end{align}
\end{subequations}
for $\epsilon \ll 1$ clearly fail to satisfy Eq.\ \eqref{eq:Ridentity}.  We can interpret the first two ratios as stating that $r$ is ``not in class 2'' and ``not in class 3,'' but the last density ratio asserts that classes 1 and 3 have equal probability.  Thus it appears that $r$ does not belong in any of the classes, which is nominally a contradiction.  In other words, Eq.\ \eqref{eq:Ridentity} ensures that the classifier yields  a valid interpretation: if $K-1$ classes have low probability, the last must have high probability.

To make this intuition more rigorous, it is useful to extend the concepts of prevalence functions and density ratios to arbitrary classifiers that we do not know \textit{a priori} to be Bayes optimal.  In particular:
\begin{definition}\label{def:inducedR}
Let $\hat C:\Gamma \times \X \to \K$ be a monotone classifier, and let $\hat C_{j,k}: \XT \to \{j,k\}$ be the induced pairwise classifiers for $j\ne k$.  Then the \textbf{induced prevalence function} $\hat \q_{j,k}(r)$ at point $r$ is the value of $\alpha_1$ at which $\hat C_{j,k}(r,\alpha)$ is discontinuous, where we define $\hat \q_{j,j}(r)=1/2$.  Moreover, the function $\hat R_{j,k}:\Gamma \to [0,\infty]$ defined by
\begin{align}
\hat R_{j,k}(r)=\frac{\hat \q_{j,k}(r)}{1-\hat \q_{j,k}(r)}
\end{align}
is the \textbf{induced density ratio}. 
\end{definition}

We remind the reader that Def.\ \ref{def:inducedR} \textit{does not imply} that $\hat \q_{j,k}(r)$ and $\hat R_{j,k}(r)$ are inherently meaningful.  As opposed to a classifier $\hat C^\star$ that we know to be Bayes optimal, we do not know in what sense the induced density ratios actually quantify information about any probability densities $P_j(r)$.  Our goal is to determine whether this is true.  To achieve this, we require the following additional property.

\begin{definition}[Self-Consistency]
Let $\hat C:\Gamma \times \X \to \K$ be a monotone classifier with induced density ratios $\hat R_{i,j}$.  Then we say that $\hat C$ is \textbf{self-consistent} if for each $r$ there exist two disjoint sets $W(r)\subset \K$ and $V(r)=\K/W$ (which depend on $r$) such that:
\begin{itemize}
\item  for all $i,j\in W(r)$, the $\hat R_{i,j}(r)$ are positive, finite, and satisfy $\hat R_{i,j}(r)=\hat R_{j,i}^{-1}(r)$; 
\item  for all $i,j,k \in W(r)$, $\hat R_{i,j}(r)\hat R_{j,k}(r)=\hat R_{i,k}(r)$; 
\item  for $i\in W(r)$, $j\in V$, $\hat R_{i,j}(r) = 0$ and $\hat R_{j,i}(r)=\infty$.
\end{itemize}
We also say that the induced density ratios $\hat R_{i,j}(r)$ are self-consistent.
\end{definition}

We can prove that self-consistency implies that a classifier agrees with the law of total probability.  However, this result requires a preliminary lemma about the structure of density ratios.

\begin{lemma}\label{lem:ridentity}
Let $\hat R_{j,k}:\Gamma \to [0,\infty]$ for $j,k\in \K$.  Then \textcolor{black}{excluding indeterminate forms,} the $\hat R_{j,k}$ can be expressed in terms of at most $K$ non-negative, bounded functions via Eq.\ \eqref{eq:relprob} if and only if the $\hat R_{j,k}(r)$ are self-consistent.  
\end{lemma}

\begin{proof}
Clearly Eq.\ \eqref{eq:relprob} implies that the density ratios are self-consistent.  Thus, it is sufficient to prove the converse for fixed values of $r$.

To do so, let there be $K^2$ numbers $\hat R_{j,k}$ for $j,k\in \K$, where we temporarily omit the dependence of $\hat R_{i,j}$ on $r$.  Assume first that the $\hat R_{j,k}$ are positive, finite, and self-consistent.  Clearly we can decompose the $\hat R_{j,k}$ into ratios of numbers $f_{j,k}$ and $g_{j,k}$ via
\begin{align}
\hat R_{j,k} = \frac{f_{j,k}}{g_{j,k}}.  
\end{align} 
The identity $\hat R_{k,j}=\hat R_{j,k}^{-1}$ implies that we can express $\hat R_{k,j}=g_{j,k}/f_{j,k}$.  Finally, observe that the $f_{j,k}$ and $g_{j,k}$ are only defined up to the same multiplicative constant.  Considering the identity
\begin{align}
\hat R_{j,k}\hat R_{k,m}=\hat R_{j,m} = \frac{f_{j,k}f_{k,m}}{g_{j,k}g_{k,m}},
\end{align}
we see that this remains true if we scale $f_{k,m}$ such that $f_{k,m}=g_{j,k}$.  Because this holds for any $j,m$, we can further restrict $f_{k,m}=P_k=g_{j,k}$.  Thus, we can express the $\hat R_{j,k}$ as ratios of at most $K$ numbers $P_j$.  

In the event that one or more of the $\hat R_{j,k}$ takes the values of $0$ or $\infty$, self-consistency implies that there exists a subset $W\subset \K$ such that the density ratios are positive and bounded.  By the previous argument, we may clearly define the $\hat R_{j,k}$ as ratios of the form $\hat R_{j,k}=P_k/P_j$ for $j,k\in W$.  Next consider $m\in V$.  The choice $P_m=0$ agrees with the assignment $\hat R_{j,m}=P_m/P_j=0$ and $\hat R_{m,j}=P_j/P_m=P_j/0=\infty$, \textcolor{black}{with the remaining density ratios indeterminate (i.e.\ taking any value).}  \qed
\end{proof}

\begin{remark}
In Lemma \ref{lem:ridentity}, we identify the set $W(r)$ as those classes for which $P_j(r)>0$, whereas $V(r)$ is the set of classes for which $P_j(r)=0$.  Thus, if $i,j\in V(r)$, then $r\notin \Gamma_{i,j}$, and we treat $\hat R_{i,j}(r)$ as density ratios associated with indeterminate forms, for which we cannot assign a meaningful interpretation.  For this reason, we may informally interpret $\hat R_{i,j}(r)$ as extensions of the true density ratios onto $\Gamma$ in the spirit of Def.\ \ref{def:qextension}.
\end{remark}

\begin{proposition}\label{prop:reverse}
Let $\hat C:\Gamma \times \X$ be a self-consistent classifier, and let $\hat R_{j,k}$ be the induced density ratios.  Then the induced prevalence functions $\hat \q_{j,k}$ satisfy Eq.\ \eqref{eq:oneconsistency}.  Moreover, if there exists a Bayes classifier $\hat C^\star$ having density ratios $\hat R_{j,k}$, then $\hat C$ is that classifier.  In particular, the inherent uncertainty $u(r)$ is given by Eq.\ \eqref{eq:MultiZacc} with $\q_{i,j}(r) = \hat \q_{i,j}(r)$.  
\end{proposition}

\begin{proof}
Consider the first claim.  Let $\chi$ be any test-prevalence in the interior of $\X$.  By the definition of $\hat R_{j,k}(r)$ and Lemma \ref{lem:ridentity}, we may decompose the $\hat R_{i,j}(r)$ into ratios of functions $\hat P_j(r)$, and in particular, we may assign $\hat P_m(r)=0$ for $m\in V(r)$.  We conclude that 
\begin{align}
\sum_{j=1}^K  \frac{\chi_j}{\sum_{k=1}^K\chi_k \hat \q_{j,k}(r)/[1-\hat \q_{j,k}(r)]} &= \sum_{j\in W(r)}  \frac{\chi_j}{\sum_{k\in W(r)}\chi_k \hat \q_{j,k}(r)/[1-\hat \q_{j,k}(r)]} \nonumber \\
 &= \sum_{j\in W(r)}  \frac{\chi_j}{\sum_{k\in W(r)}\chi_k \hat R_{j,k}(r)} \nonumber \\
 &=  \sum_{j\in W(r)}  \frac{\hat P_j(r)\chi_j}{\sum_{k\in W(r)}\chi_k \hat P_k(r)} = 1.  \label{eq:reverseq}
\end{align}
%where the $\hat P_j(r)$ are numbers that depend on $r$ and whose existence is guaranteed by Lemma \ref{lem:ridentity}. % Moreover, in Eq.\ \eqref{eq:reverseq}, for $j,k\in V(r)$, we may set $\hat P_j(r)=\hat P_k(r)=0$.  

Next, let $\hat \Phi_{i,j}=\{r:i,j\in W(r)\}$, and let $\hat \Psi_{i,j}=\{r:i\in W(r), j\in V(r), \hat R_{i,j}(r)=0 \}$, from which we construct the set
\begin{align}
\hat \Gamma_i = \bigcup_{j, j\ne i}\left[ \hat \Phi_{i,j}\cup \hat \Psi_{i,j}\right].
\end{align}
Clearly $\hat \Gamma_i$ is the set of points for which we assign $\hat P_i(r) > 0$, and $\hat \Gamma_{i,j} = \hat \Gamma_i \cup \hat \Gamma_j$ is the set of points for which $\hat P_i(r) > 0$ and/or $\hat P_j(r) > 0$.  Moreover, the set $\hat \Gamma = \bigcup_i \hat \Gamma_i$ is the set of points $r$ on which at least one $\hat P_j(r) > 0$.  

Finally, let $\hat C^\star:\hat \Gamma \times \X$ be any Bayes classifier constructed from densities $P_j(r)$ whose supports are $\hat \Gamma_j$ and having density ratios $R_{i,j}(r) = \hat R_{i,j}(r)$, where $R_{i,j}: \hat \Gamma_{i,j} \to [0,\infty]$.  Clearly this classifier is self-consistent and has the same sets $W(r)$ and $V(r)$ as $\hat C$.  Hence $\hat C$ is the Bayes classifier.  \qed
\end{proof}

A fundamental limitation of Prop.\ \ref{prop:reverse} is the fact that it assumes the existence of a Bayes classifier whose densities have the specified supports.  While this is not likely to be an issue in practice, it is deeply unsatisfying from a theoretical standpoint.  The problem arises from the fact that Lemma \ref{lem:ridentity} simply states that the $\hat P_j(r)$ exist and are not unique, but it does not tell us how to construct such functions.  By itself, the non-uniqueness is not inherently surprising: a single Bayes classifiers can be associated with different sets of density functions having the same set of density ratios.  But without a prescription for constructing the $\hat P_j(r)$, Prop. \ref{prop:reverse} does not guarantee that the induced densities are measurable or normalizable.  The former issue likely involves a complicated foray into measure theory, which is beyond the scope of this work.  However, it is straightforward to show that self-consistency \textit{does not} imply normalization, an issue that we pursue in the next section.

\subsubsection{Normalization for Classifiers}

Let $\Gamma_i$ be the support of $P_i(r)$.  Definition \ref{def:densityratio} implies the identity
\begin{align}
\int_{\Gamma_i}\d r \,\,  R_{i,j}(r) P_i(r) + \int_{\Gamma_j / \Gamma_i} \!\!\!\d r \,\, P_j(r) = \int_{\Gamma_j} \d r \,\, P_j(r)= 1.\label{eq:normalizationidentity}
\end{align}
A key implication of Eq.\ \eqref{eq:normalizationidentity} is that once the $R_{i,j}$ are known, the densities $P_j$ can no longer be treated as independent.  But the fact that the $P_j(r)$ are normalized \textit{also} imposes structure on the $R_{i,j}$.  For example, it is not possible that $R_{i,j}(r)>1$ for all $r\in \Gamma_i$, since this implies that $P_j(r)$ is not normalized.  

Unfortunately, we are unaware of a simple and explicit condition on $\hat R_{i,j}$ ensuring that the $P_j(r)$ are normalized.  Clearly the $(\hat R_{i,j} - 1)$ must be negative on a set of finite measure with respect to $P_i(r)$.  Moreover, if $\Gamma_i= \Gamma_j=\Gamma_k$, it is not possible for $R_{i,j} < R_{i,k}$ for all values of $r$, since this would imply that at least one of the densities $P_i$, $P_j$, or $P_k$ is not normalized.  Whether these constitute \textit{sufficient} conditions is unknown to us except in special cases.  Thus, we simply state the normalization condition as follows.
\begin{definition}
Let $\hat C$ be a  self-consistent classifier.  Let $\hat R_{i,j}$ be the induced density ratios expressed in terms of measurable functions $\hat P_j(r)$ having supports $\hat \Gamma_i$.  Then we say that the classifier is \textbf{normalizable} if Eq.\ \eqref{eq:normalizationidentity} is satisfied under the substitution $R_{i,j} \to \hat R_{i,j}$, $P_j \to \hat P_j$, and $\Gamma_j \to \hat \Gamma_j$.  
\end{definition}

\subsubsection{Binary Classifiers Revisited}

While self-consistency and normalizability are necessary and sufficient conditions guaranteeing that a classifier induces meaningful density functions $P_j(r)$, the latter criterion assumes measurability of the $\hat P_j(r)$.  We are unaware of general conditions guaranteeing this except in the case of binary classifiers.  
\begin{proposition}\label{prop:binary_existence}
Let $\hat C:\Gamma \times \XT \to \{1,2\}$ be a self-consistent binary classifier, where $\Gamma \subset \mathbb R^n$ is a compact set for $n \ge 1$.  If the induced density ratio $\hat R_{1,2}$ is positive, bounded, and has the properties that $ 0 < \hat R_{1,2} < 1$ on a set of finite Lebesgue measure, and moreover, $ 1 < \hat R_{1,2} < M < \infty$ on a set of finite Lebesgue measure, then there exist density functions $P_1(r)$ and $P_2(r)$ for which $\hat C$ is the Bayes classifier.
\end{proposition}

\begin{proof}
We proceed by construction.  Let $\Lambda_0=\{r: \hat R_{1,2}(r) = 1\}$, $\Lambda_1 = \{r: \hat R_{1,2}(r) < 1\}$, and $\Lambda_2 = \{r: \hat R_{1,2}(r) > 1\}$.  Let $\hat P_1(r) = \alpha_m$ for $r\in \Lambda_m$, where $\alpha_m$ are constants and $m=0,1,2$.  Normalization of $\hat P_2(r)$ implies the identity
\begin{align}
\alpha_1 \int_{\Lambda_1} \d r \,\, (\hat R_{1,2}(r) - 1) + \alpha_2 \int_{\Lambda_2} \d r \,\, (\hat R_{1,2}(r)-1) = 0. \label{eq:linsys}
\end{align}
By compactness, both integrals are finite, so that Eq.\ \eqref{eq:linsys} is valid linear equation, and it also implies that $\alpha_1$ and $\alpha_2$ have the same sign.  Moreover, normalization of $\hat P_1$ implies 
\begin{align}
\sum_{m=0}^2 \alpha_m \mu(\Lambda_m) = 1
\end{align}
where $\mu(\Lambda_m)$ is the Lebesgue measure of the $\Lambda_m$.  This leads to a system of equations for $\alpha_m$, which yields a normalized and strictly positive $\hat P_1(r)$.  Moreover, $\hat P_2(r) = R_{1,2}(r) \hat P_1(r)$ is normalized as a result of Eq.\ \eqref{eq:linsys}.  Thus, there exist density functions for which $\hat C:\Gamma \times \XT \to \{1,2\}$ is the Bayes classifier for all test prevalence values $\chi \in \XT$.  
\end{proof}

The challenge in extending Prop.\ \ref{prop:binary_existence} to the multiclass case arises from the fact that the density $\hat P_1(r)$ must jointly satisfy the normalization condition for all $j > 1$.  One could in principle consider a compact domain $\Gamma$ as before and subdivide it into increasingly small regions, treating $\hat P_1$ as constant on each of these.  This process yields a system of equations that should have a solution, given sufficient granularity of the domain.  However, we are unaware of a method for demonstrating that the resulting density function is non-negative, although it can likely be normalized.  Thus, it seems that the joint requirements of positivity and normalizability renders the multiclass problem more difficult.

\section{Practical Issues Through the Lens of Examples}
\label{sec:examps}

The analysis of the previous sections assumes that it is possible to evaluate classifiers across an arbitrarily large number of affine prevalence values.  In practice, however, one is restricted to a finite number of such computations, as well as limited sampling of training and test data.  This has several implications for real-world testing of monotonicity and self-consistency.  The purpose of this section is to explore such issues while illustrating main concepts in the context of examples.

\subsubsection{Examples of Induced Density Ratios}

We begin with a series of three- and four-class examples that illustrate certain properties of Lemma \ref{lem:ridentity} and Proposition \ref{prop:reverse}.  The goal of these examples is to address the practical issue of identifying the sets $W$ and $V$, especially in a multiclass setting with large $K$.  Here we assume that the $\hat R_{i,j}$ are both known exactly and self-consistent, focusing on the fact that these must be evaluated pointwise for each $r$ in an empirical test population.  In this context, it is convenient to define an \textbf{induced density matrix} $\hat \bR$ whose $(i,j)$th element is $\hat R_{i,j}$. We temporarily omit any dependence on the variable $r$, since this can be treated as fixed.  

Consider first $\K = \{1,2,3\}$.  If $\hat P_j > 0$ for all $j$, then clearly $\hat R_{i,j}$ satisfy Eq.\ \eqref{eq:Ridentity}.  A more interesting case occurs when only $\hat P_1=0$, which yields
\begin{align}
\hat \bR = \begin{bmatrix}
1 & \infty & \infty\\
 0 & 1 & a\\ 
 0& 1/a & 1
\end{bmatrix}
\end{align}
for some $a> 0$.  Another interesting case occurs when only $\hat P_3 > 0$, so that one finds
\begin{align}
\hat \bR = \begin{bmatrix}
1 & ? & \infty\\
 ? & 1 & \infty\\ 
 0& 0 & 1
\end{bmatrix}
\end{align}
where $?$ stands for an indeterminate form whose value we cannot anticipate.  

We may continue in this way for a four-class setting.  From left to right, the following equations illustrate the cases $\hat P_1=0$, $\hat P_1=\hat P_2=0$, and $\hat P_1=\hat P_2=\hat P_3=0$, with all remaining densities positive.  One finds
\begin{align}
\hat \bR = \begin{bmatrix}
1 & \infty & \infty & \infty \\
 0 & 1 & a & b\\ 
 0& 1/a & 1 & c \\
 0 & 1/b & 1/c & 1
\end{bmatrix},
&&\hat \bR = \begin{bmatrix}
1 & ? & \infty & \infty \\
 ? & 1 & \infty & \infty \\ 
 0& 0 & 1 & c \\
 0 & 0 & 1/c & 1
\end{bmatrix},
&&\hat \bR = \begin{bmatrix}
1 & ? & ? & \infty \\
 ? & 1 & ? & \infty\\ 
 ?& ? & 1 & \infty \\
 0 & 0 & 0 & 1
\end{bmatrix},
\end{align}
for some positive numbers $a$, $b$, and $c$.  

These examples illustrate that  $\hat \bR$ can be put into a standard form as follows.
\begin{definition}\label{def:standard}
Let $\hat \bR$ be the induced density ratio matrix.  We say that $\hat \bR$ is in \textbf{standard form} if it is in a block form
\begin{align}
\hat \bR = \begin{bmatrix}
{\boldsymbol {\rm A}} & \boldsymbol \infty \\
{\boldsymbol 0} & {\boldsymbol {\rm B}}
\end{bmatrix},
\end{align}
where the elements of the square matrix ${\boldsymbol {\rm B}}$ satisfy Eq.\ \eqref{eq:Ridentity}, $\boldsymbol \infty$ is a matrix of whose elements are all $\infty$, and ${\boldsymbol 0}$ is a matrix whose elements are all zero.  
\end{definition}

It trivial to show that $W$ is the set of all indices belonging to any ordered pair of ${\boldsymbol {\rm B}}$, whereas the elements of $V$ are those indices belonging to any ordered pair of ${\boldsymbol {\rm A}}$.  Thus, the density ratios are self-consistent if indices can be relabeled such that the matrix $\hat \bR$ is in standard form, which is a convenient way of checking for this property.  We also speculate that one can inductively relabel indices of an arbitrary $\hat \bR$ to put in into standard form, although this task is left for future work.  In the examples of the next sections, the standard forms can be constructed by recourse to the above examples.  

%Finally it is important to note that in the examples above, the matrix $\hat \bR$ for which (i) all elements above the main diagonal are $\infty$, and (ii) all elements below the main diagonal are $0$ can technically be interpreted as self consistent.  However, 

\subsection{Illustration of Construction \ref{constr:multi}: Gaussian Distributions}

We next demonstrate some basic properties of class switching and Construction \ref{constr:multi}.  Consider a three-class problem for which the conditional PDFs are given by
\begin{align}
P_j(x,y)&=\frac{1}{2\pi|\Sigma_j|^{1/2}}\exp\left[-\frac{1}{2} (x-\mu_{j,x},y-\mu_{j,y}) \Sigma_j^{-1} \binom{x-\mu_{j,x}}{y-\mu_{j,y}} \right], \label{eq:gausspdfs}
\end{align}
where $\Sigma_j$ is a $2 \times 2$ covariance matrix, and $\mu_{j,x}$ and $\mu_{j,y}$ are fixed parameters.  For $j\in\{1,2,3\}$, these are chosen to be
\begin{subequations}
\begin{align}
\Sigma_1 &= \begin{pmatrix}
0.49 & 0\\
0 & 0.09 
\end{pmatrix}, &
\Sigma_2 &= \begin{pmatrix}
0.16 & 0\\
0 & 0.64 
\end{pmatrix}, & 
\Sigma_3 &= \begin{pmatrix}
0.25 & 0\\
0 & 0.01 
\end{pmatrix},\label{eq:covariance}
\end{align}
and
\begin{align}
\mu_{1,x} &=0 & \mu_{2,x} &=0 & \mu_{3,x} &=1 , \label{eq:means1} \\
\mu_{1,y} &=1 & \mu_{2,y} &=-1 & \mu_{3,y} &=0. \label{eq:means2}
\end{align}
\end{subequations}

Figure \ref{fig:ternary} shows a finite number of pairwise boundary sets and classifiers, along with the corresponding three-class Bayes classifier for the test prevalence of $\chi=(1/3,1/3,1/3)$.  For the pairwise classifiers (top two plots and bottom left plot), the fixed point denoted by an $x$ clearly changes class at a single value of the affine prevalence.  Observe also that according to Construction \ref{constr:multi}, it is possible to express the three-class Bayes classifiers in terms of its pairwise counterparts.  This means that the boundary separating classes in the latter should be a composite of the pairwise boundary sets, which is shown in the bottom-right plot of Fig.\ \ref{fig:ternary}.  Because the densities are given explicitly, this example directly verifies Construction \ref{constr:multi}.

\begin{figure}
\includegraphics[width=14cm]{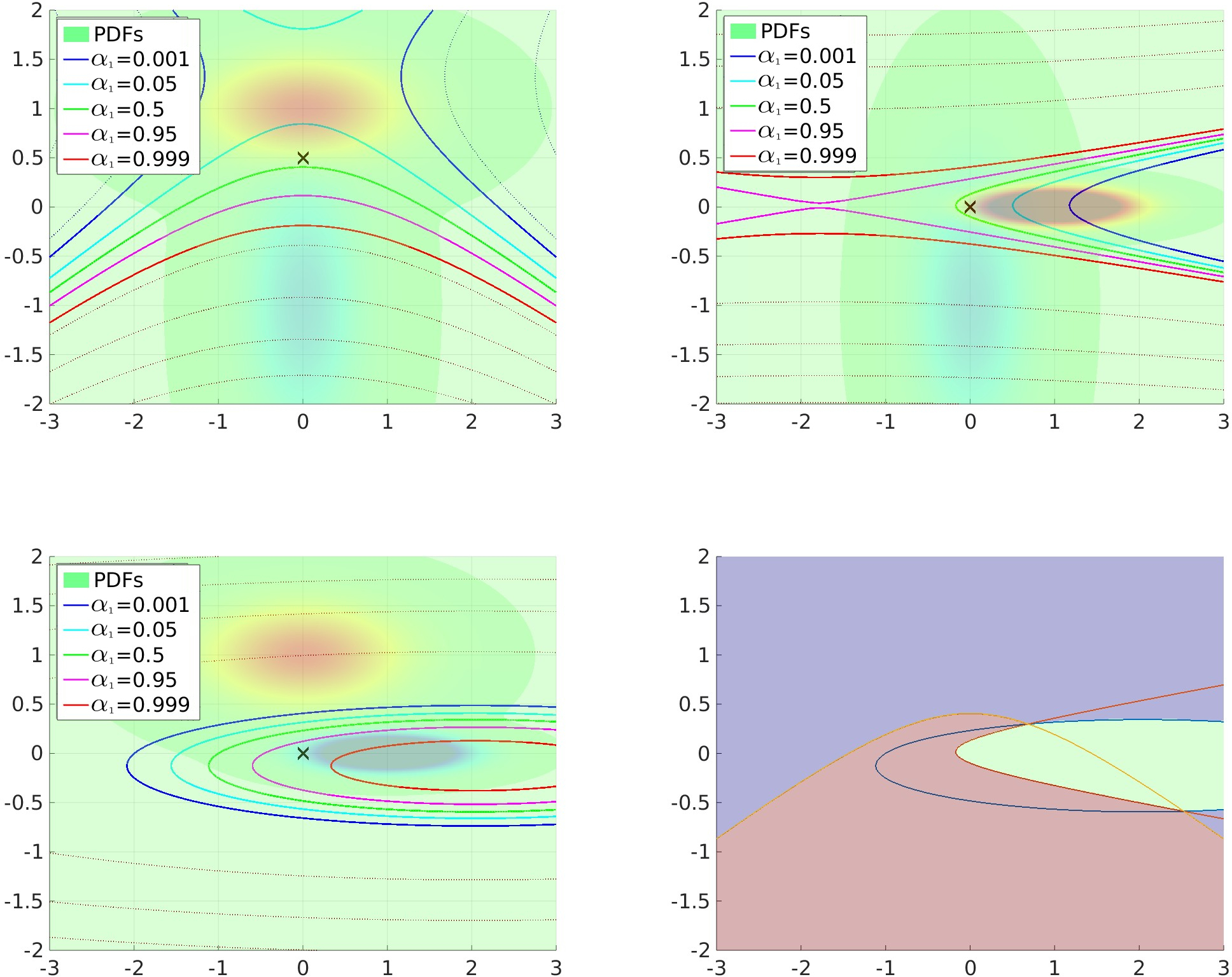}\caption{Pairwise monotone classifiers for the probability distributions given by Eqs.\ \eqref{eq:gausspdfs}--\eqref{eq:means2}.  The top-left, top-right, and bottom-left plots show pairwise classification boundaries associated with $\q_{1,2}(r)$, $\q_{3,2}(r)$, and $\q_{1,3}(r)$, respectively.  The solid lines correspond to level-sets as indicated in the legend, whereas the dotted contours show additional (unlabeled) contours.  The bottom-right plot shows the Bayes optimal classifier for $q=(1/3,1/3,1/3)$.  Note that these domains are bounded by pairwise relative probabiliy level-sets, as predicted by the theory of Sec.\ \ref{sec:level-set}.}\label{fig:ternary}
\end{figure}

\subsection{Character Recognition and Self-Consistency}

Next we consider a canonical image classification problem using the MNIST character recognition dataset \cite{MNIST}.  This example illustrates the  challenges of using Construction \ref{constr:genmulti} when we do not know \textit{a priori} that $\hat C$ is Bayes optimal and can only test for self-consistency on a  discrete grid of affine prevalence values.  

\subsubsection{MNIST Dataset and Neural Network Architecture}
\label{subsub:MNIST}

The full dataset is comprised of 60,000 training images of handwritten numerals `0' to `9', along with 10,000 validation images.  (We use single-quotes to distinguish handwritten numerals from numbers.)  Here we consider the restricted three-class problem of identifying the numerals `1', `2', and `7', which can be difficult to distinguish visually.  We let $\omega$ denote a handwritten numeral, $r(\omega)$ the corresponding grid of $28 \times 28$ pixel values associated with its image, and $C(\omega)=k\in \{1,2,3\}$ the true class of the image under the mapping $1 \mapsto$ `1',  $2 \mapsto$ `2', and $3 \mapsto$ `7'.  Properties of the training and test data are listed in Table \ref{tab:dataprops}.  Note that $s_k$ denotes the number of samples in the $k$th class.  
{\rowcolors{2}{blue!10}{blue!20}
\begin{table}[h!]
\centering
\begin{tabular}{|c | c | c | c| c | c |} 
 \hline
 {\bf Numeral} & $C(\omega)=k$ & {\bf Training} $s_k$ & {\bf Training} $\chi_k$ & {\bf Test} $s_k$ & {\bf Test} $\chi_k$ \\ [0.5ex] 
 \hline
 `1' & 1 & 6742 & 0.3555 & 1135 & 0.3552 \\ 
 `2' & 2 & 5958 & 0.3142 & 1032 & 0.3230 \\
 `7' & 3 & 6265 & 0.3303 & 1028 & 0.3218 \\ [1ex] 
 \hline
\end{tabular}
\caption{Properties of the MNIST data used in the handwriting-recognition example.  The second column $C(\omega)$ denotes the mathematically assigned label associated with each numeral.  The notation $s_k$ stands for the number of samples in the $k$th class, whereas $\chi_k$ denotes the prevalence of the training or test population.  }
\label{tab:dataprops}
\end{table}

\begin{figure}\begin{center}
\includegraphics[width=12cm]{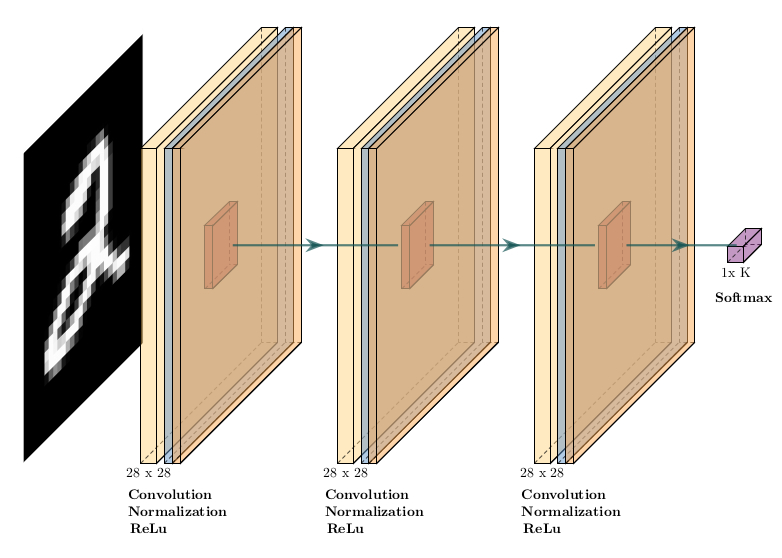}\caption{Schematic of the neural network architectures used in this work.  The input is a $28\times 28$ pixel image with 256 grayscale levels.  This feeds into a convolutional layer with a convolution width of $5$ pixels and zero-padding, which ensures that the output (large yellow rectangle with smaller orange ractangle) is the same size as the input.  Next we apply a normalization layer (blue), followed by a rectified linear unit (ReLu) layer (light orange).  The output of this is fed to another set of layers with an identical structure, followed by another.  The end-result is fed into a softmax layer.  Reference \cite{NNPlotter} was used to generate this figure. }\label{fig:NN}
\end{center}
\end{figure}

To perform character recognition, construct a neural network with the architecture illustrated in Fig.\ \ref{fig:NN}.  We choose $\hat C(r,q)$ to be
\begin{align}
\hat C(r,q) = \argmax_j \left[\mathcal S(r,q) \right],
\end{align} 
where $\mathcal S:\Gamma \times \X \to [0,1]^K$ denotes the softmax output function of the NN, and $\mathcal S_m$ denotes the $m$th element of $\mathcal S$.  Given finite training data $\Pi_{tr}$, our goal is to minimize the empirical objective (or loss) function
\begin{align}
\mathcal L\left[\hat C(r);\Pi_{tr}, q\right] = \sum_{k=1}^K\frac{q_k}{s_k} \sum_{j=1}^{s_k} \mathbb I \left[C(\omega_{j,k}) \ne \hat C\left(r(\omega_{j,k})\right)\right],  \label{eq:emploss}
\end{align}
where $\omega_{j,k}$ is the $j$th element from the $k$th class and $q$ is temporarily treated as a fixed parameter.  While it could in principle be natural to consider the cross-entropy loss and directly quantify the $\Pr[C|r]$, Eq.\ \eqref{eq:emploss} allows us to: (i) directly approximate the (affine) prevalence-weighted expected loss; and (ii) thereby approximate the deterministic Bayes classifier $\hat C^\star(r,q)$.  Moreover, the failure of $H$-consistency in Ref.\ \cite{Amudhan} motivates us to consider Eq.\ \eqref{eq:emploss} and regularization thereof.  

A key problem with Eq.\ \eqref{eq:emploss} is that it is not differentiable.  As a result, backpropagation is not possible \cite{DNN}, so that optimization methods such as stochastic gradient descent cannot be used \cite{DNN}.  One solution to this is to consider a homotopy-type approach wherein we consider a sequence of objective functions that converge Eq.\ \eqref{eq:emploss} as a scaling parameter $\sigma\to 0$.  See Refs.\ \cite{Homotopy1,Homotopy2,Homotopy3,Homotopy4,Homotopy5} for general background on homotopy approaches, as well as \cite{PartI,Amudhan,Regularized} for uses and variations of this method in the context of ML.  Here we consider the homotopy objective function as a composition of several functions:
\begin{subequations}
\begin{align}
y_{j,k,m}&=\mathcal S_m\left(r(\omega_{j,k})\right) \label{eq:softmax} \\
h(x;\sigma)&= (1/2)[1-\tanh(x/\sigma)],    \\
g(x;\sigma)&= (1/2)[\tanh\left((x-0.5)/\sigma \right) + 1],\\
\Delta y_{j,k,m} &= y_{j,k,k} - y_{j,k,m}, \\  
\mathcal L\left[\mathcal S;\Pi_{tr},q,\sigma \right] &=    \sum_{k=1}^K \frac{q_k}{s_k} \sum_{j=1}^{s_k}  g\left[  \sum_{m=1}^K g\left( h(\Delta y_{j,k,m};\sigma) ;\sigma \right) ;\sigma \right].              \label{eq:homotopyobjective}                  %\sum_{k=1}^K \frac{q_k}{s_k} \sum_{j=1}^{s_k} 
\end{align}
\end{subequations}
It is straightforward to show that as $\sigma \to 0$, Eq.\ \eqref{eq:homotopyobjective} converges to Eq.\ \eqref{eq:emploss}.  

For each pair of numerals, we trained the pairwise classifiers according to the pseudocode in Algorithm \ref{alg:nntrain}, where $\mathcal S^{\{j,3\}}$ is the trained neural network corresponding to the $j$th prevalence $\bq_j$.  In each optimization step, we used stochastic gradient descent with a learning rate of 0.01 and 15 epochs (no mini-batching).  In the pseudocode below, the notation $\mathcal E_\chi$ denotes the cross-entropy, which is used only to initialize the NN, and the notation $\argmin_{\mathcal S} \left[ L | \mathcal X \right ]$ means optimization of $L$ over $\mathcal S$ starting at point $\mathcal X$.  Given this family of NNs, the relative probabilities of belonging to any given class can be estimated using Eq.\ \eqref{eq:MultiZacc} and the pairwise training prevalences.  Note, however, that the prevalence was only sampled on a grid with a spacing of $\Delta q_1=0.01$, so that we can only provide lower and upper bounds on the classification accuracy.  This corresponds to $\epsilon_{\rm num}$ as discussed in Sec.\ \ref{subsec:imp4UQ}.  

\begin{algorithm}
\caption{Neural-Network Training Algorithm}\label{alg:nntrain}
\begin{algorithmic}[1]
\Procedure{Pairwise Neural Network Training}{}
\State $\bsigma \gets \{1,2,4\}$
\State $\bq \gets \{0.01,0.02,...,0.99\}$
\State $L \gets \mathcal L\left[\mathcal S; \Pi_{tr},(\bq_1,1-\bq_1),\bsigma_1 \right] + \mathcal E_\chi$
\State $\mathcal S^{\{0\}} \gets $ Initialize with random weights
\State $\mathcal S^{\{1,1\}} \gets \argmin_{\mathcal S} \left[ L | \mathcal S^{\{0\}} \right]$
\State \emph{loop}:
\For {j=1:99}
	\For {k=2:3}
		\State $\mathcal S^{\{j,k\}} = \argmin_{\mathcal S}\left[ \mathcal L\left[\mathcal S; \Pi_{tr},(\bq_j,1-\bq_j),\bsigma_k \right] | \mathcal S^{\{j,k-1\}} \right]$
	\EndFor
	\State $\mathcal S^{\{j+1,1\}} \gets \mathcal S^{\{j,3\}}$
\EndFor
\EndProcedure
\end{algorithmic}
\end{algorithm}

\subsubsection{Self-Consistency}
\label{subsub:consistency}

Because the affine prevalence is only sampled on a finite grid, we can at best conclude that the induced prevalence function $\hat \q_{i,j}(r)$ [or equivalently, $\hat R_{i,j}(r)$] for a fixed $r$ is bounded from below and above.  Moreover, our pairwise classifiers occasionally violate monotonicity, switching classes multiple times; see also Ref.\ \cite{Langford05}.  Thus, we are limited to the knowledge that
\begin{align}
\hat R_{i,j}(r) \in [\hat R_{i,j}^\ell(r),\hat R_{i,j}^h(r)] \label{eq:binninguncertainty}
\end{align}  
for some low and high bounds $\hat R_{i,j}^\ell(r)$ and $\hat R_{i,j}^h(r)$ depending on $r$ and the density with which the affine prevalences are sampled.  For values of $r$ at which a pairwise classifier is non-monotone, we take the bounds to be the first and last value of the induced density ratio (as a function of the affine prevalence) at which the class switches.  

A key implication of Eq.\ \eqref{eq:binninguncertainty} is that we cannot in general assign a unique value of inherent uncertainty $u(r)$ to any given point.  This is a form of binning uncertainty discussed in Sec.\ \ref{subsec:imp4UQ}.  However, we can still estimate a range of uncertainties and test for the possibility of self-consistency.  We make this precise via the following definitions.

\begin{definition}[Empirically Self-Consistent Classifier]
Let $\hat C:\Gamma \times \X \to \K$ be a monotone classifier, and let the induced density ratios $\hat R_{i,j}(r)$ be determined to within some interval $\hat R_{i,j}(r) \in [\hat R_{i,j}^\ell(r),\hat R_{i,j}^h(r)]$.  We say that the classifier is \textbf{empirically self-consistent at a point} $r$ if there exist some values of $\tilde R_{i,j}\in [\hat R_{i,j}^\ell(r),\hat R_{i,j}^h(r)]$ that are self-consistent.  If there are no such values of $\tilde R_{i,j}$, we say that the classifier is \textbf{inconsistent at} $r$. \label{def:emp_consistency}
\end{definition}
\begin{definition}
Let $\hat C:\Gamma \times X$ be a monotone classifier.  Fix $r$, let $\hat R_{i,j}\in [\hat R_{i,j}^\ell(r),\hat R_{i,j}^h(r)]$ be the induced density ratios, and construct
\begin{align}
\Psi(r)=(\tilde R_{1,1}(r),\tilde R_{1,2}(r),...,\tilde R_{1,K},\tilde R_{2,3}(r),...,\tilde R_{K-1,K}(r))
\end{align}
as a collection of values $\tilde R_{i,j}(r)\in [\hat R_{i,j}^\ell(r),\tilde R_{i,j}^h(r)]$ that are self-consistent at point $r$.  We refer to the set $\mathbb F(r)=\{\Psi(r)\}$ of all such collections $\Psi(r)$ as the \textbf{feasible set of interpretations of the classifier} (or more simply, the feasible set) at point $r$.    
\end{definition}

When the induced density ratios are specified up to intervals of the form given by Eq.\ \eqref{eq:binninguncertainty}, the structure and existence of feasible sets can be understood by recourse to geometric arguments.  Figure \ref{fig:feasible} illustrates this idea.  It shows two intervals associated with $\hat R_{1,2}$ and $\hat R_{2,3}$ on perpendicular axes.  The shaded box corresponds to all pairs of induced density ratios for which products of the form $\hat R_{1,2}\hat R_{2,3}$ could be feasible.  For this to be the case, however, such products must equal $\hat R_{1,3}$.  Note also that values of constant $\hat R_{1,2}\hat R_{2,3}$ correspond to hyperbolas, which we can parameterize by $\hat R_{1,3}$.  Thus $\hat R_{1,3}$ on some interval can be interpreted as a domain in the figure bounded by two hyperbolas, and the feasible set only exists when this domain intersects the shaded box.

Given a feasible set $\mathbb F(r)$, the task of computing $u(r)$ is ultimately a modeling choice. In the examples that follow, we define $u(r)$ to be the arithmetic mean of the inherent uncertainties computed on the vertices of domains such as those in Fig.\ \ref{fig:feasible}.  A deeper analysis of such tasks is left as an open problem. 

%It is important to note that empirical self-consistency is a pointwise property.  The feasible set $\mathbb S(r)$ depends on $r$.  

In general, we are not aware of methods for guaranteeing that a real-world monotone classifier will be empirically self-consistent on all of $\Gamma$.  Thus, it seems reasonable to assume that in practice, there will exist $r$ such that $\mathbb F(r)=\emptyset$.  In fact, we propose this as an important metric for characterizing the quality of a ML algorithm: a better classifier is one with a higher-degree of empirical self-consistency as measured in terms of the number of points for which this condition holds, all else being equal.  

%Self consistency is a critical property because it implies that Eq.\ \eqref{eq:oneconsistency} is satisfied, for example.  This ensures that there is a valid probabilistic interpretation underlying the classifier.  It is important to note, however, that this consistency is \textit{pointwise}; 

\begin{figure}
\begin{center}\includegraphics[width=12cm]{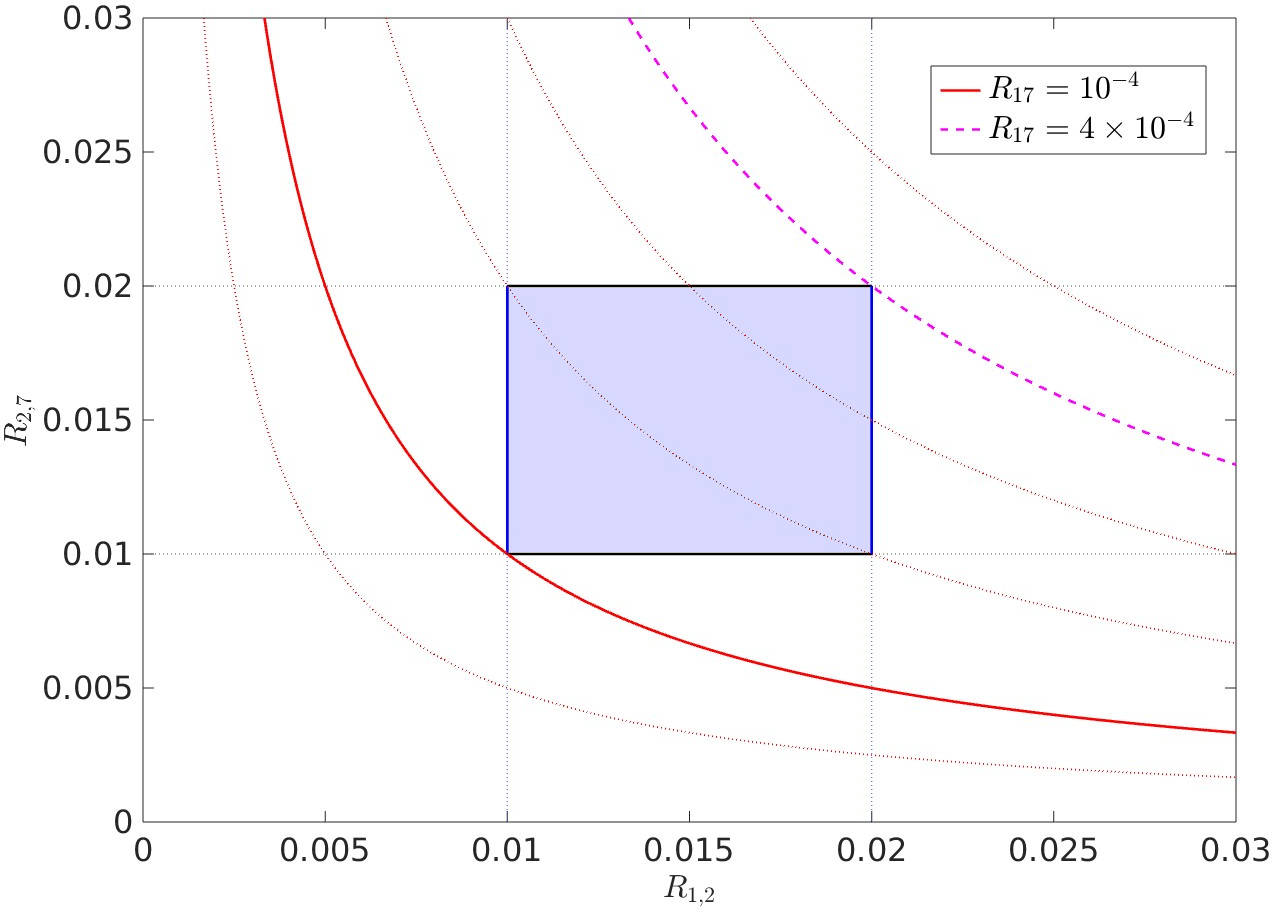}\end{center}\caption{Illustration of feasibility and empirical self-consistency in a three-class setting.  The horizontal and vertical axes correspond to values of $R_{1,2}$ and $R_{2,7}$, while the hyperbolas are curves of constant $R_{1,2}R_{2,7}$.  The shaded box shows an example of binning uncertainty in the values associated with $R_{1,2}$ and $R_{1,7}$.  For a fixed value of $r$, the triple $\Psi(r)=(R_{1,2},R_{2,7},R_{1,7})$ is only consistent if $R_{2,7}=R_{1,2}R_{2,7}$.    If a monotone classifier indicates that $R_{1,7}$ is bounded by some values $[R_{1,7}^\ell,R_{1,7}^h]$, then it is consistent if and only if the domain bounded by the corresponding hyperbolas intersects the box.  Moreover, that intersection is the feasible set.  In the examples that follow, we define $u(r)$ to be the arithmetic mean of the inherent uncertainties computed on the vertices of such domains. } \label{fig:feasible}
\end{figure}

%Even when a classifier is self-consistent, there is typically not a unique way to assign $u(r)$ unless the feasible set collapses to a single point.  When this does not happen, it is not \textit{a priori} clear how each element of $\mathbb S(r)$ should contribute to $u(r)$.  In the three-class example that follows, it is possible to show that the $\mathbb S(r)$ are either (i) the intersections of rectangles with domains bounded by hyperbolas or (ii) unbounded versions of such objects.  These domains have at most four vertices, and we take the corresponding estimate of $u(r)$ to be the arithmetic mean of the inherent uncertainties estimated at these points.  See Fig.\ \ref{fig:feasible} for more details.  In general, the task of quantifying the estimating the \textit{binning uncertainty} associated with a feasible set $\mathbb S(r)$ is an open problem, which we leave for future work. 

\begin{figure}
\includegraphics[width=12cm]{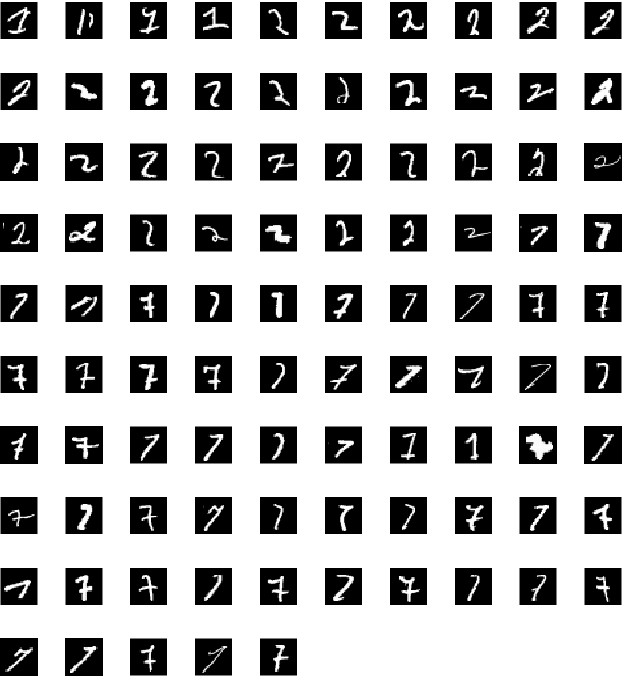}\caption{Digits `1,' `2,' and `7' from the MNIST training set that were determined to be inconsistent with the law of total probability according to Def.\ \ref{def:emp_consistency}.  }\label{fig:inconsistent}
\end{figure}
  
\begin{figure}
\includegraphics[width=12cm]{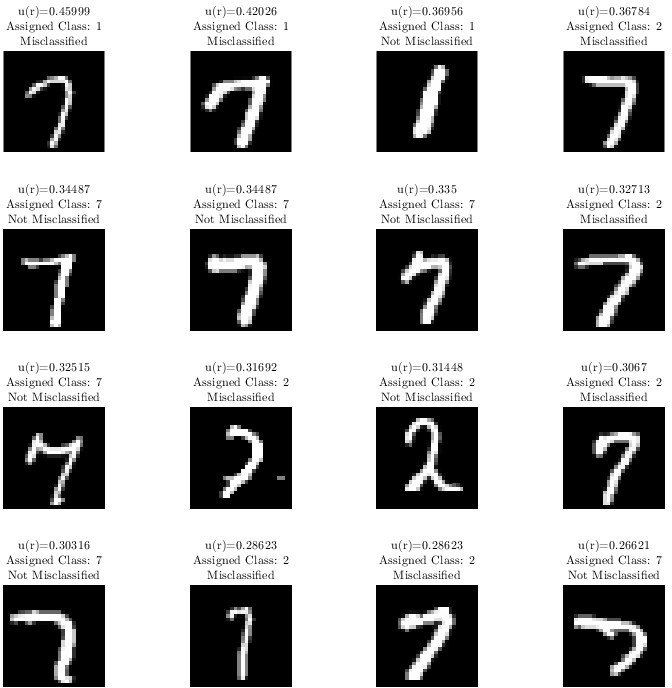}\caption{Digits `1,' `2,' and `7' from the MNIST training set that have a high uncertainty in the class labels. }\label{fig:lowprobs}
\end{figure}  
  
\begin{figure}
\includegraphics[width=12cm]{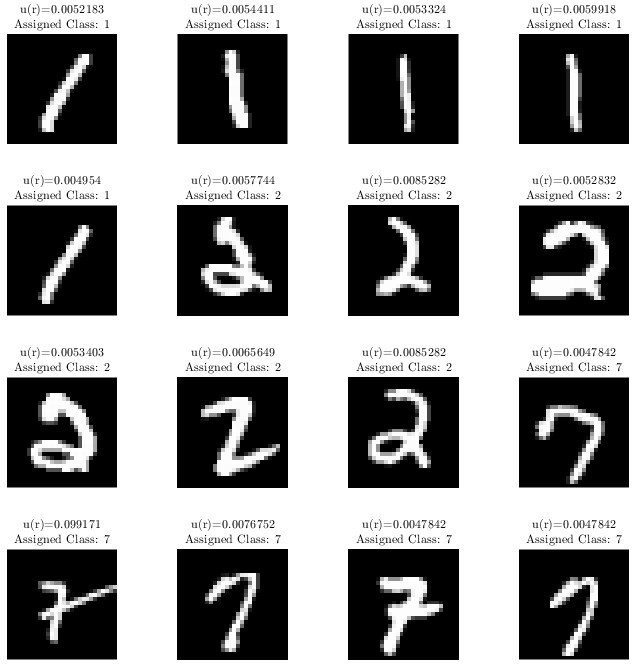}\caption{Digits `1,' `2,' and `7' from the MNIST training set with a high confidence of being classified correctly.}\label{fig:highprobs}
\end{figure}  
  
\subsubsection{Numerical Results}  
\label{subsub:results}  
  
Table \ref{tab:performance} summarizes the outcomes of our numerical experiments on the MNIST data.  After training the family of pairwise classifiers and computing bounds on the density ratios, our analysis rejected 95 training points as inconsistent with the law of total probability. Figure \ref{fig:inconsistent} shows all 95 of these images.  Such data are inputs for which we cannot meaningfully assign a class because the algorithm yields results that contradict the law of total probability; recall Eq.\ \eqref{eq:oneconsistency}.  It is comforting that many of these images bear no resemblance whatsoever to any of the numerals `1', `2', or `7'.  This suggests that the concept of self-consistency is in fact useful for identifying data that cannot be analyzed by the NN.  Figures \ref{fig:lowprobs} and \ref{fig:highprobs} show sample images from the training set that have low and high probabilities of being classified correctly.
  
Table \ref{tab:performance} also shows the fraction of points within certain local accuracy ranges and the fraction thereof that were classified correctly, given the training and test prevalence values $\chi$.  Such results indicate that the probabilistic interpretation of the classifier, while not perfect, is largely in line with the actual performance of the NN.  In other words, our model appears to be calibrated \cite{Calibration1,Calibration2}, although it perhaps overestimates its accuracy slightly.  See also Figs. \ref{fig:lowprobs} and \ref{fig:highprobs}.
  
{\rowcolors{2}{blue!10}{blue!20}
\begin{table}[h!]
\centering
\begin{tabular}{|c | c | c | c |} 
 \hline
 {\bf Training Data $Z(r)$ Range} & {\bf \# of Samples} & {\bf \# Classified Correctly} & {\bf \% Correct} \\ [0.5ex] 
 \hline
 Inconsistent & 95 & 0 & 0 \\
 $0.5 \le Z(r) < 0.6 $ & 2 & 0 & 0 \\ 
 $0.6 \le Z(r) < 0.7$ & 11 & 7 & 0.6364 \\
 $0.7 \le Z(r) < 0.8$ & 0 & -- & --  \\
 $0.8 \le Z(r) < 0.9$ & 53 & 50 & 0.9434 \\
 $0.9 \le Z(r) < 0.98$ & 340 & 336 & 0.9882 \\
 All consistent samples & 18870 & 18817 & 0.9972 \\ 
\hline \hline
{\bf Test Data $Z(r)$ Range} & & & \\
 Inconsistent & 0 & 0 & 0 \\
 $0.5 \le Z(r) < 0.6 $ & 6 & 1 & 0.1667 \\ 
 $0.6 \le Z(r) < 0.7$ & 2 & 1 & 0.5 \\
 $0.7 \le Z(r) < 0.8$ & 0 & -- & --  \\
 $0.8 \le Z(r) < 0.9$ & 17 & 15 & 0.8824 \\
 $0.9 \le Z(r) < 0.98$ & 56 & 43 & 0.7679 \\
 All consistent samples & 3195 & 3174 & 0.9934 \\ 
 [1ex] 
 \hline
\end{tabular}
\caption{Characterization of the performance of the NN whose training is discussed in Sec. \ref{subsub:MNIST}.}
\label{tab:performance}
\end{table}

\section{Discussion}
\label{sec:discussion}

\subsection{On the Interpretation of Prevalence as an Affine Parameter}
\label{subsec:affine}

Figures \ref{fig:binex} and \ref{fig:ternary} demonstrate examples in which the boundary sets $B^\star(q)$ deform continuously as $q$ is varied.  Recalling Eq.\ \eqref{eq:bset}, the density ratios are constant along a given curve, so that it is compelling to think of any one of these as an axis in some curvilinear coordinate system.  Moreover, in this coordinate system, the affine prevalence $q\in [0,1]$ is one of the natural independent variables, with the other being some arc-length parameter $\theta$ that orders positions along a level-set.  Thus, a level set could be expressed as some function $r(q,\theta)$.  

If the density ratio $R_{i,j}$ is constant along any such curve, then it stands to reason that $Q(r,\chi)$ still  varies.  To understand the implications of this, consider the density function 
\begin{align}
\frac{Q(r,\chi)}{P_1(r)} = \chi_1 + \chi_2 R_{1,2}(r)  
\end{align}
of a test population $\Omega(\chi)$, for example.  Provided we could sample $Q(r,\chi)$ along such a curve (e.g\ given sufficient training or test data), it would be possible to reconstruct $P_1(r)$, since the density ratio $R_{1,2}(r)$ can be determined via class switching.

\subsection{Considerations for Optimization: Homotopy Methods, H-Consistency, and Cross Entropy}

The analysis herein assumes that it is possible to minimize the prevalence-weighted 0-1 loss.  As has been noted, however, empirical versions of this objective function are often discontinuous, which requires some form of regularization if derivative-based optimization methods are to be used.  This is the motivation for Eq.\ \eqref{eq:homotopyobjective}, and similar methods have been proposed in Refs.\ \cite{PartI,Amudhan,Regularized}.  A fundamental question, however, is how to extend homotopy-type methods to other loss functions, and whether this is even necessary.  

Recent work on cross-entropy minimization suggests compelling routes for addressing this question.  The emerging understanding of H-consistency states that the classifiers minimizing infinite-sample surrogate loss functions can and often do yield the Bayes optimal classifier \cite{Hconsistency1,Hconsistency2,Hconsistency3,Hconsistency4}, although this depends heavily on the hypothesis set \cite{Amudhan}.  The cross-entropy loss function in particular is known to be H-consistent \cite{CrossEntropy1}, and moreover this objective function is well adapted to many classifiers such as neural networks.  Thus, it is likely that the analysis presented herein could be adapted to cross-entropy minimization provided this loss function can be suitably modified.

To see what this might look like, consider that the empirical version of cross entropy can be expressed as
\begin{align}
\mathcal E_{\times}(\Pi_{tr},\phi) = - \sum_{k=1}^K  \sum_{j=1}^{s_k} \log(y_{j,k,k}(\phi)) \label{eq:basiccrossentropy}
\end{align}
where $y_{j,k,k}$ is as defined in Eq.\ \eqref{eq:softmax}, and we assume that the underlying softmax function is parameterized in terms of $\phi$.  [That is, $\phi$ determines the NN achitecture, for example.]  We note in particular that the prevalence of the training population is implicit in Eq.\ \eqref{eq:basiccrossentropy}, since one could, for example, change the number of elements in a given class, and thereby decrease its contribution to the cross-entropy calculation.  We can make this dependence more explicit by observing that for $N$ total training points and any training prevalence, Eq.\ \eqref{eq:basiccrossentropy} is equivalent to
\begin{align}
\mathcal E_{\times}(\Pi_{tr},\phi) &=  \sum_{k=1}^K \frac{s_k}{s_k N}\sum_{j=1}^{s_k} -\log(y_{j,k,k}(\phi))=  \sum_{k=1}^K \frac{\bar q_k}{s_k} \sum_{j=1}^{s_k} -\log(y_{j,k,k}(\phi)),  \label{eq:bettercrossentropy}
\end{align}
where $\bar q_k=s_k/N$ is the training prevalence of the $k$th class.  As this mirrors the structure of Eq.\ \eqref{eq:homotopyobjective}, we hypothesize that $\bar q_k$ can be replaced by an affine prevalence $q_k$ in Eq.\ \eqref{eq:bettercrossentropy} to yield a prevalence-weighted version fo the cross entropy objective function.  Determining whether this is H-consistent with the prevalence-weighted 0-1 loss is thus an interesting open question.  

In the context of our analysis, cross-entropy is also of inherent interest.  The goal of minimizing this objective function is essentially to reconstruct the probabilities $\Pr[C|r]$, e.g.\ in terms of the softmax function.  In light of Eq.\ \eqref{eq:condprob}, this is equivalent to $\Pr[r|C]$ if $\Pr[C]$ is known.  Thus, the class-switching property of a classifier yields testable predictions.  In particular, if the prevalence-weighted cross-entropy does indeed yield $\Pr[C|r]$ for all values of $q$, then the following conjecture should be true.
\begin{conjecture}
Let $S(r,q)$ be a binary softmax function whose elements $S_1(r,q)$ and $S_2(r,q)$ minimize Eq.\ \eqref{eq:bettercrossentropy} in the infinite sample limit, where $q$ is an affine prevalence.  Then for a fixed $r$, the value of $q$ at which $S_1(r,q)=S_2(r,q)$ is ${\q(r)=(\q_{1,2}(r),1-\q_{1,2}(r))}$.  Moreover, ${S_1(r,(q_1,1-q_1))=(q_1/q_1')S_1(r,(q_1',1-q_1'))}$.
\end{conjecture}
Validation of this claim using Eq.\ \eqref{eq:bettercrossentropy} is left as future work.   

\subsection{Unification of Generative and Discriminative Learning}

Generative classifiers are often taken to be those that directly model the distributions $\Pr[r|C]$, whereas discriminative classifiers are often said to model $\Pr[C|r]$ \cite{Disc3}; see also Refs.\ \cite{Disc1,Disc2} for related perspectives.  By Eq.\ \eqref{eq:condprob}, it is obvious that these two are equivalent given the prevalence $\Pr[C]$.  For this reason, we view the class switching property as unifying discriminative and generative classifiers, since it shows \textit{how} to deduce properties of the $\Pr[r|C]$ given \textit{any} optimized $\hat C(r,q):\Gamma \times \X \to \mathbb K$.  In other words, the analysis herein demonstrates how the process of training a ML model can be understood as a form of statistical regression on the densities $\Pr[r|C]$.

%In this context, one could rightfully claim that we have left many facets of this unification unaddressed.  We have not actually extracted the densities $P_j(r)$, only their relative magnitudes.  We have only considered tasks wherein the underlying probability space does not have Dirac masses.  And perhaps most importantly, we have not considered issues associated with rates of convergence.  For example, how quickly do the resulting $P_j(r)$ converge to their true counterparts, and can such convergence be achieved, e.g.\ without homotopy methods or for finite $\sigma$?
%
%While many of these remain open questions beyond the scope of this manuscript, we heuristically consider the first: can the $P_j(r)$ be extracted from the relative probabilities?  We speculate that the answer is yes.  In particular, the level sets $B^\star(q)$ are the collection of points on which the relative probabilities are constant, but the \textit{absolute} fraction of points along a level-set can vary.  In other words, $Q(r)$ is not constant for $r\in B^\star(q)$.  Thus, if one can deduce methods for extracting the test density along level sets, it should be possible to fully estimate the $P_j(r)$.  We anticipate that this is only possible when the boundary sets have simple topologies and/or sufficient regularity.  

\subsection{Additonal Limitations and Open Directions}

A fundamental problem in machine learning, and in particular with methods that estimate class probabilities, is calibration \cite{Calibration1,Calibration2}.  In particular, this requires ensuring that not only are the model predictions accurate, but also that model predictions of classification accuracy are correct.  In other words, a model that under or overestimates is predictive capability can be as bad as one that makes poor predictions, especially in real-world settings \cite{CalibrationApplied}.  The work presented herein does not address such questions, and it is likely that prevalence-weighted loss functions can suffer from miscalibration, especially when trained on finite data.  Finding ways to address this problem will be critical for realizing and making use of the tools developed in this manuscript.  

{\it Acknowledgements:}  The authors thank Dr.\ Bradley Alpert for extremely helpful feedback on this manuscript, which was instrumental in clarifying the core ideas of our work.

%\appendix
%
%\section{Proofs of Auxiliary Results}
%
%In this section we prove Lemma \ref{lem:ridentity}. 
%

%
\bibliographystyle{elsarticle-num}
\bibliography{AI_Gen}

\begin{thebibliography}{10}
\expandafter\ifx\csname url\endcsname\relax
  \def\url#1{\texttt{#1}}\fi
\expandafter\ifx\csname urlprefix\endcsname\relax\def\urlprefix{URL }\fi
\expandafter\ifx\csname href\endcsname\relax
  \def\href#1#2{#2} \def\path#1{#1}\fi

\bibitem{UQ_Survey1}
W.~He, Z.~Jiang, T.~Xiao, Z.~Xu, Y.~Li,
  \href{https://arxiv.org/abs/2302.13425}{A survey on uncertainty
  quantification methods for deep learning} (2025).
\newblock \href {http://arxiv.org/abs/2302.13425} {\path{arXiv:2302.13425}}.
\newline\urlprefix\url{https://arxiv.org/abs/2302.13425}

\bibitem{UQ_Survey2}
J.~Gawlikowski, C.~R.~N. Tassi, M.~Ali, J.~Lee, M.~Humt, J.~Feng, A.~Kruspe,
  R.~Triebel, P.~Jung, R.~Roscher, M.~Shahzad, W.~Yang, R.~Bamler, X.~X. Zhu,
  \href{https://doi.org/10.1007/s10462-023-10562-9}{A survey of uncertainty in
  deep neural networks}, Artificial Intelligence Review 56~(1) (2023)
  1513--1589.
\newblock \href {https://doi.org/10.1007/s10462-023-10562-9}
  {\path{doi:10.1007/s10462-023-10562-9}}.
\newline\urlprefix\url{https://doi.org/10.1007/s10462-023-10562-9}

\bibitem{UQ_Survey3}
J.~Mena, O.~Pujol, J.~Vitri{\`a},
  \href{https://api.semanticscholar.org/CorpusID:241258024}{A survey on
  uncertainty estimation in deep learning classification systems from a
  bayesian perspective}, ACM Computing Surveys (CSUR) 54 (2021) 1 -- 35.
\newline\urlprefix\url{https://api.semanticscholar.org/CorpusID:241258024}

\bibitem{UQ_Survey4}
M.~Abdar, F.~Pourpanah, S.~Hussain, D.~Rezazadegan, L.~Liu, M.~Ghavamzadeh,
  P.~Fieguth, X.~Cao, A.~Khosravi, U.~R. Acharya, V.~Makarenkov, S.~Nahavandi,
  \href{http://dx.doi.org/10.1016/j.inffus.2021.05.008}{A review of uncertainty
  quantification in deep learning: Techniques, applications and challenges},
  Information Fusion 76 (2021) 243–297.
\newblock \href {https://doi.org/10.1016/j.inffus.2021.05.008}
  {\path{doi:10.1016/j.inffus.2021.05.008}}.
\newline\urlprefix\url{http://dx.doi.org/10.1016/j.inffus.2021.05.008}

\bibitem{MLBook}
B.~Bohn, J.~Garcke, M.~Griebel,
  \href{https://epubs.siam.org/doi/abs/10.1137/1.9781611977882}{Algorithmic
  Mathematics in Machine Learning}, Society for Industrial and Applied
  Mathematics, Philadelphia, PA, 2024.
\newblock \href
  {http://arxiv.org/abs/https://epubs.siam.org/doi/pdf/10.1137/1.9781611977882}
  {\path{arXiv:https://epubs.siam.org/doi/pdf/10.1137/1.9781611977882}}, \href
  {https://doi.org/10.1137/1.9781611977882}
  {\path{doi:10.1137/1.9781611977882}}.
\newline\urlprefix\url{https://epubs.siam.org/doi/abs/10.1137/1.9781611977882}

\bibitem{Hconsistency1}
P.~Awasthi, A.~Mao, M.~Mohri, Y.~Zhong,
  \href{https://proceedings.mlr.press/v162/awasthi22c.html}{H-consistency
  bounds for surrogate loss minimizers}, in: K.~Chaudhuri, S.~Jegelka, L.~Song,
  C.~Szepesvari, G.~Niu, S.~Sabato (Eds.), Proceedings of the 39th
  International Conference on Machine Learning, Vol. 162 of Proceedings of
  Machine Learning Research, PMLR, 2022, pp. 1117--1174.
\newline\urlprefix\url{https://proceedings.mlr.press/v162/awasthi22c.html}

\bibitem{Hconsistency2}
P.~Awasthi, A.~Mao, M.~Mohri, Y.~Zhong,
  \href{https://openreview.net/forum?id=06OVtS901hF}{Multi-class h-consistency
  bounds}, in: A.~H. Oh, A.~Agarwal, D.~Belgrave, K.~Cho (Eds.), Advances in
  Neural Information Processing Systems, 2022.
\newline\urlprefix\url{https://openreview.net/forum?id=06OVtS901hF}

\bibitem{Hconsistency3}
A.~Mao, M.~Mohri, Y.~Zhong, \href{https://arxiv.org/abs/2407.13722}{Enhanced
  h-consistency bounds} (2024).
\newblock \href {http://arxiv.org/abs/2407.13722} {\path{arXiv:2407.13722}}.
\newline\urlprefix\url{https://arxiv.org/abs/2407.13722}

\bibitem{Hconsistency4}
P.~Awasthi, N.~Frank, A.~Mao, M.~Mohri, Y.~Zhong, Calibration and consistency
  of adversarial surrogate losses, in: M.~Ranzato, A.~Beygelzimer, Y.~Dauphin,
  P.~Liang, J.~W. Vaughan (Eds.), Advances in Neural Information Processing
  Systems, Vol.~34, Curran Associates, Inc., 2021, pp. 9804--9815.

\bibitem{CalibrationReview}
T.~Silva~Filho, H.~Song, M.~Perello-Nieto, R.~Santos-Rodriguez, M.~Kull,
  P.~Flach, Classifier calibration: a survey on how to assess and improve
  predicted class probabilities, Machine Learning 112~(9) (2023) 3211--3260.

\bibitem{Calibration1}
C.~Wang, Calibration in deep learning: A survey of the state-of-the-art (2024).
\newblock \href {http://arxiv.org/abs/2308.01222} {\path{arXiv:2308.01222}}.

\bibitem{Calibration2}
F.~M. Ojeda, M.~L. Jansen, A.~Thiéry, S.~Blankenberg, C.~Weimar, M.~Schmid,
  A.~Ziegler, Calibrating machine learning approaches for probability
  estimation: A comprehensive comparison, Statistics in Medicine 42~(29) (2023)
  5451--5478.

\bibitem{CalibrationApplied}
B.~Van~Calster, D.~J. McLernon, M.~van Smeden, L.~Wynants, E.~W. Steyerberg,
  P.~Bossuyt, G.~S. Collins, P.~Macaskill, K.~G.~M. Moons, B.~Van Calster,
  M.~van Smeden, A.~Vickers, O.~b. o. T. G. E.~d. tests, p.~m. o. t.~S.
  initiative, Calibration: the achilles heel of predictive analytics, BMC
  Medicine 17~(1) (2019) 230.

\bibitem{Calibration2016}
P.~A. Flach, Classifier Calibration, Springer US, Boston, MA, 2016, pp. 1--8.

\bibitem{Calibration2017}
C.~Guo, G.~Pleiss, Y.~Sun, K.~Q. Weinberger, On calibration of modern neural
  networks, in: D.~Precup, Y.~W. Teh (Eds.), Proceedings of the 34th
  International Conference on Machine Learning, Vol.~70 of Proceedings of
  Machine Learning Research, PMLR, 2017, pp. 1321--1330.

\bibitem{Calibration2018}
M.~P. Naeini, G.~F. Cooper, Binary classifier calibration using an ensemble of
  piecewise linear regression models, KNOWLEDGE AND INFORMATION SYSTEMS 54~(1)
  (2018) 151--170.

\bibitem{SmithUQ}
R.~Smith, Uncertainty Quantification: Theory, Implementation, and Applications,
  Computational Science and Engineering, Society for Industrial and Applied
  Mathematics, 2013.

\bibitem{CrossEntropy1}
A.~Mao, M.~Mohri, Y.~Zhong,
  \href{https://arxiv.org/abs/2304.07288}{Cross-entropy loss functions:
  Theoretical analysis and applications} (2023).
\newblock \href {http://arxiv.org/abs/2304.07288} {\path{arXiv:2304.07288}}.
\newline\urlprefix\url{https://arxiv.org/abs/2304.07288}

\bibitem{SandiaUQML}
D.~J. Stracuzzi, M.~G. Chen, M.~C. Darling, M.~G. Peterson, C.~Vollmer,
  Uncertainty quantification for machine learning, Tech. rep., Sandia National
  Lab.(SNL-NM), Albuquerque, NM (United States) (2017).

\bibitem{Langford05}
J.~Langford, B.~Zadrozny, Estimating class membership probabilities using
  classifier learners, in: International Workshop on Artificial Intelligence
  and Statistics, PMLR, 2005, pp. 198--205.

\bibitem{Disc3}
T.~Mitchell, Machine Learning, McGraw Hill series in computer science, McGraw
  Hill, 2017.

\bibitem{QuantificationBook}
A.~Esuli, A.~Fabris, A.~Moreo, F.~Sebastiani, Methods for Learning to Quantify,
  Springer International Publishing, Cham, 2023, pp. 55--85.

\bibitem{Quantification71}
A.~H. Murphy, Scalar and vector partitions of the probability score: Part i.
  two-state situation, Journal of Applied Meteorology and Climatology 11~(2)
  (1972) 273 -- 282.

\bibitem{Quantification2005}
G.~Forman, Counting positives accurately despite inaccurate classification, in:
  Proceedings of the 16th European Conference on Machine Learning, ECML'05,
  Springer-Verlag, Berlin, Heidelberg, 2005, p. 564–575.

\bibitem{Quantification2009}
K.~Takahashi, H.~Takamura, M.~Okumura, Direct estimation of class membership
  probabilities for multiclass classification using multiple scores, Knowledge
  and Information Systems 19~(2) (2009) 185--210.

\bibitem{Quantification2016}
A.~Firat, Unified framework for quantification (2016).
\newblock \href {http://arxiv.org/abs/1606.00868} {\path{arXiv:1606.00868}}.

\bibitem{Quantification2017}
P.~Gonz\'{a}lez, A.~Casta\~{n}o, N.~V. Chawla, J.~J.~D. Coz, A review on
  quantification learning, ACM Comput. Surv. 50~(5) (Sep. 2017).

\bibitem{Quantification2018}
Z.~Lipton, Y.-X. Wang, A.~Smola, Detecting and correcting for label shift with
  black box predictors (02 2018).

\bibitem{Quantification2022}
M.~Bunse, Unification of algorithms for quantification and unfolding, in:
  INFORMATIK 2022, Gesellschaft f{\"u}r Informatik, Bonn, 2022, pp. 459--468.

\bibitem{Quantification2024}
J.~Luo, F.~Hong, J.~Yao, B.~Han, Y.~Zhang, Y.~Wang, Revive re-weighting in
  imbalanced learning by density ratio estimation, in: A.~Globerson, L.~Mackey,
  D.~Belgrave, A.~Fan, U.~Paquet, J.~Tomczak, C.~Zhang (Eds.), Advances in
  Neural Information Processing Systems, Vol.~37, Curran Associates, Inc.,
  2024, pp. 79909--79934.

\bibitem{Zhang04}
T.~Zhang, Statistical analysis of some multi-category large margin
  classification methods, J. Mach. Learn. Res. 5 (2004) 1225–1251.

\bibitem{Halck}
O.~M. Halck, Using hard classifiers to estimate conditional class
  probabilities, in: T.~Elomaa, H.~Mannila, H.~Toivonen (Eds.), Machine
  Learning: ECML 2002, Springer Berlin Heidelberg, Berlin, Heidelberg, 2002,
  pp. 124--134.

\bibitem{MultiGen}
K.~Takahashi, H.~Takamura, M.~Okumura, Direct estimation of class membership
  probabilities for multiclass classification using multiple scores, Knowledge
  and Information Systems 19~(2) (2009) 185--210.
\newblock \href {https://doi.org/10.1007/s10115-008-0165-z}
  {\path{doi:10.1007/s10115-008-0165-z}}.

\bibitem{Tao}
T.~Tao, An Introduction to Measure Theory, Graduate studies in mathematics,
  American Mathematical Society, 2013.

\bibitem{MLBookCh1}
\href{https://epubs.siam.org/doi/abs/10.1137/1.9781611977882.ch1}{Chapter 1:
  The Basics of Machine Learning}, pp. 1--16.
\newblock \href
  {http://arxiv.org/abs/https://epubs.siam.org/doi/pdf/10.1137/1.9781611977882.ch1}
  {\path{arXiv:https://epubs.siam.org/doi/pdf/10.1137/1.9781611977882.ch1}},
  \href {https://doi.org/10.1137/1.9781611977882.ch1}
  {\path{doi:10.1137/1.9781611977882.ch1}}.
\newline\urlprefix\url{https://epubs.siam.org/doi/abs/10.1137/1.9781611977882.ch1}

\bibitem{SVM1}
M.~Awad, R.~Khanna, Support Vector Machines for Classification, Apress,
  Berkeley, CA, 2015, pp. 39--66.

\bibitem{SVM2}
N.~Cristianini, J.~Shawe-Taylor, An Introduction to Support Vector Machines and
  Other Kernel-based Learning Methods, Cambridge University Press, 2000.
\newblock \href {https://doi.org/10.1017/CBO9780511801389}
  {\path{doi:10.1017/CBO9780511801389}}.

\bibitem{SVM3}
K.~P. Bennett, C.~Campbell, Support vector machines: Hype or hallelujah?,
  SIGKDD Explor. Newsl. 2~(2) (2000) 1–13.
\newblock \href {https://doi.org/10.1145/380995.380999}
  {\path{doi:10.1145/380995.380999}}.

\bibitem{BNN1}
M.~Magris, A.~Iosifidis,
  \href{https://doi.org/10.1007/s10462-023-10443-1}{Bayesian learning for
  neural networks: an algorithmic survey}, Artificial Intelligence Review
  56~(10) (2023) 11773--11823.
\newblock \href {https://doi.org/10.1007/s10462-023-10443-1}
  {\path{doi:10.1007/s10462-023-10443-1}}.
\newline\urlprefix\url{https://doi.org/10.1007/s10462-023-10443-1}

\bibitem{BNN2}
J.~Arbel, K.~Pitas, M.~Vladimirova, V.~Fortuin,
  \href{https://arxiv.org/abs/2309.16314}{A primer on bayesian neural networks:
  Review and debates} (2023).
\newblock \href {http://arxiv.org/abs/2309.16314} {\path{arXiv:2309.16314}}.
\newline\urlprefix\url{https://arxiv.org/abs/2309.16314}

\bibitem{DNN}
\href{https://epubs.siam.org/doi/abs/10.1137/1.9781611977882.ch6}{Chapter 6:
  Deep Neural Networks}, pp. 87--108.
\newblock \href
  {http://arxiv.org/abs/https://epubs.siam.org/doi/pdf/10.1137/1.9781611977882.ch6}
  {\path{arXiv:https://epubs.siam.org/doi/pdf/10.1137/1.9781611977882.ch6}},
  \href {https://doi.org/10.1137/1.9781611977882.ch6}
  {\path{doi:10.1137/1.9781611977882.ch6}}.
\newline\urlprefix\url{https://epubs.siam.org/doi/abs/10.1137/1.9781611977882.ch6}

\bibitem{PartI}
P.~N. Patrone, R.~A. Binder, C.~S. Forconi, A.~M. Moormann, A.~J. Kearsley,
  \href{https://arxiv.org/abs/2309.00645}{Analysis of diagnostics (part i):
  Prevalence, uncertainty quantification, and machine learning} (2024).
\newblock \href {http://arxiv.org/abs/2309.00645} {\path{arXiv:2309.00645}}.
\newline\urlprefix\url{https://arxiv.org/abs/2309.00645}

\bibitem{LiebLoss}
E.~Lieb, M.~Loss, Analysis, Crm Proceedings \& Lecture Notes, American
  Mathematical Society, 2001.

\bibitem{Inherent}
K.~Fukunaga,
  \href{https://www.sciencedirect.com/science/article/pii/B9780080478654500077}{Chapter
  1 - introduction}, in: K.~Fukunaga (Ed.), Introduction to Statistical Pattern
  Recognition (Second Edition), second edition Edition, Academic Press, Boston,
  1990, pp. 1--10.
\newblock \href
  {https://doi.org/https://doi.org/10.1016/B978-0-08-047865-4.50007-7}
  {\path{doi:https://doi.org/10.1016/B978-0-08-047865-4.50007-7}}.
\newline\urlprefix\url{https://www.sciencedirect.com/science/article/pii/B9780080478654500077}

\bibitem{Ward1}
D.~Needell, R.~Ward, N.~Srebro, Stochastic gradient descent, weighted sampling,
  and the randomized kaczmarz algorithm, in: Z.~Ghahramani, M.~Welling,
  C.~Cortes, N.~Lawrence, K.~Weinberger (Eds.), Advances in Neural Information
  Processing Systems, Vol.~27, Curran Associates, Inc., 2014.

\bibitem{Ward2}
R.~Ward, X.~Wu, L.~Bottou, Adagrad stepsizes: Sharp convergence over nonconvex
  landscapes, Journal of Machine Learning Research 21~(219) (2020) 1--30.

\bibitem{BoisvertUQ}
A.~Dienstfrey, R.~Boisvert, Uncertainty Quantification in Scientific Computing
  : 10th IFIP WG2.5Working Conference, WoCoUQ 2011, Boulder, CO, USA, August
  1-4, 2011, 2012.

\bibitem{GladenRogan}
Z.~Lang, J.~Reiczigel,
  \href{https://www.sciencedirect.com/science/article/pii/S0167587713002936}{Confidence
  limits for prevalence of disease adjusted for estimated sensitivity and
  specificity}, Preventive Veterinary Medicine 113~(1) (2014) 13--22.
\newblock \href
  {https://doi.org/https://doi.org/10.1016/j.prevetmed.2013.09.015}
  {\path{doi:https://doi.org/10.1016/j.prevetmed.2013.09.015}}.
\newline\urlprefix\url{https://www.sciencedirect.com/science/article/pii/S0167587713002936}

\bibitem{Patrone22_2}
P.~Patrone, A.~Kearsley, Minimizing uncertainty in prevalence estimates (2022).
\newblock \href {http://arxiv.org/abs/2203.12792} {\path{arXiv:2203.12792}}.

\bibitem{OldPrevOpt}
C.~Brownie, J.-P. Habicht, \href{http://www.jstor.org/stable/2530910}{Selecting
  a screening cut-off point or diagnostic criterion for comparing prevalences
  of disease}, Biometrics 40~(3) (1984) 675--684.
\newline\urlprefix\url{http://www.jstor.org/stable/2530910}

\bibitem{RW}
C.~Rasmussen, C.~Williams, Gaussian Processes for Machine Learning, Adaptative
  computation and machine learning series, University Press Group Limited,
  2006.

\bibitem{Patrone21_1}
P.~N. Patrone, A.~J. Kearsley, {Classification under uncertainty: data analysis
  for diagnostic antibody testing}, Mathematical Medicine and Biology: A
  Journal of the IMA 38~(3) (2021) 396--416.
\newblock \href {https://doi.org/10.1093/imammb/dqab007}
  {\path{doi:10.1093/imammb/dqab007}}.

\bibitem{Metrics1}
K.~Dembczy{\'{n}}ski, W.~Kot{\l}owski, O.~Koyejo, N.~Natarajan,
  \href{https://proceedings.mlr.press/v70/dembczynski17a.html}{Consistency
  analysis for binary classification revisited}, in: D.~Precup, Y.~W. Teh
  (Eds.), Proceedings of the 34th International Conference on Machine Learning,
  Vol.~70 of Proceedings of Machine Learning Research, PMLR, 2017, pp.
  961--969.
\newline\urlprefix\url{https://proceedings.mlr.press/v70/dembczynski17a.html}

\bibitem{Metrics2}
H.~Narasimhan, H.~Ramaswamy, A.~Saha, S.~Agarwal,
  \href{https://proceedings.mlr.press/v37/narasimhanb15.html}{Consistent
  multiclass algorithms for complex performance measures}, in: F.~Bach, D.~Blei
  (Eds.), Proceedings of the 32nd International Conference on Machine Learning,
  Vol.~37 of Proceedings of Machine Learning Research, PMLR, Lille, France,
  2015, pp. 2398--2407.
\newline\urlprefix\url{https://proceedings.mlr.press/v37/narasimhanb15.html}

\bibitem{Metrics3}
O.~Koyejo, N.~Natarajan, P.~Ravikumar, I.~S. Dhillon,
  \href{https://api.semanticscholar.org/CorpusID:6680925}{Consistent binary
  classification with generalized performance metrics}, in: Neural Information
  Processing Systems, 2014.
\newline\urlprefix\url{https://api.semanticscholar.org/CorpusID:6680925}

\bibitem{Metrics4}
O.~Koyejo, N.~Natarajan, P.~Ravikumar, I.~S. Dhillon,
  \href{https://api.semanticscholar.org/CorpusID:6348016}{Consistent multilabel
  classification}, in: Neural Information Processing Systems, 2015.
\newline\urlprefix\url{https://api.semanticscholar.org/CorpusID:6348016}

\bibitem{Metrics5}
H.~Narasimhan, P.~Kar, P.~Jain,
  \href{https://api.semanticscholar.org/CorpusID:1424505}{Optimizing
  non-decomposable performance measures: A tale of two classes}, in:
  International Conference on Machine Learning, 2015.
\newline\urlprefix\url{https://api.semanticscholar.org/CorpusID:1424505}

\bibitem{Metrics6}
A.~S. Jadhav,
  \href{https://www.sciencedirect.com/science/article/pii/S0957417420302153}{A
  novel weighted tpr-tnr measure to assess performance of the classifiers},
  Expert Systems with Applications 152 (2020) 113391.
\newblock \href {https://doi.org/https://doi.org/10.1016/j.eswa.2020.113391}
  {\path{doi:https://doi.org/10.1016/j.eswa.2020.113391}}.
\newline\urlprefix\url{https://www.sciencedirect.com/science/article/pii/S0957417420302153}

\bibitem{MultiStrategies}
E.~L. Allwein, R.~E. Schapire, Y.~Singer,
  \href{https://doi.org/10.1162/15324430152733133}{Reducing multiclass to
  binary: a unifying approach for margin classifiers}, J. Mach. Learn. Res. 1
  (2001) 113–141.
\newblock \href {https://doi.org/10.1162/15324430152733133}
  {\path{doi:10.1162/15324430152733133}}.
\newline\urlprefix\url{https://doi.org/10.1162/15324430152733133}

\bibitem{OneVOne1}
T.~Hastie, R.~Tibshirani, Classification by pairwise coupling, in: M.~Jordan,
  M.~Kearns, S.~Solla (Eds.), Advances in Neural Information Processing
  Systems, Vol.~10, MIT Press, 1997.

\bibitem{OneVOne2}
T.-F. Wu, C.-J. Lin, R.~C. Weng, Probability estimates for multi-class
  classification by pairwise coupling, J. Mach. Learn. Res. 5 (2004)
  975–1005.

\bibitem{OneVOne3}
B.~Zadrozny, Reducing multiclass to binary by coupling probability estimates,
  in: T.~Dietterich, S.~Becker, Z.~Ghahramani (Eds.), Advances in Neural
  Information Processing Systems, Vol.~14, MIT Press, 2001.

\bibitem{OneVOne4}
L.~Zhou, Q.~Wang, H.~Fujita,
  \href{https://www.sciencedirect.com/science/article/pii/S1566253516301439}{One
  versus one multi-class classification fusion using optimizing decision
  directed acyclic graph for predicting listing status of companies},
  Information Fusion 36 (2017) 80--89.
\newblock \href {https://doi.org/https://doi.org/10.1016/j.inffus.2016.11.009}
  {\path{doi:https://doi.org/10.1016/j.inffus.2016.11.009}}.
\newline\urlprefix\url{https://www.sciencedirect.com/science/article/pii/S1566253516301439}

\bibitem{Reduction}
A.~Beygelzimer, V.~Dani, T.~Hayes, J.~Langford, B.~Zadrozny,
  \href{https://doi.org/10.1145/1102351.1102358}{Error limiting reductions
  between classification tasks}, in: Proceedings of the 22nd International
  Conference on Machine Learning, ICML '05, Association for Computing
  Machinery, New York, NY, USA, 2005, p. 49–56.
\newblock \href {https://doi.org/10.1145/1102351.1102358}
  {\path{doi:10.1145/1102351.1102358}}.
\newline\urlprefix\url{https://doi.org/10.1145/1102351.1102358}

\bibitem{OneVAll1}
R.~Rifkin, A.~Klautau, In defense of one-vs-all classification, Journal of
  Machine Learning Research 5 (2004) 101--141.

\bibitem{OneVAll2}
R.~Rifkin, A.~Klautau, In defense of one-vs-all classification, Journal of
  Machine Learning Research 5 (2004) 101--141.

\bibitem{Pairwise1}
M.~Krzyśko, W.~Wołyński,
  \href{https://www.sciencedirect.com/science/article/pii/S0377221708009752}{New
  variants of pairwise classification}, European Journal of Operational
  Research 199~(2) (2009) 512--519.
\newblock \href {https://doi.org/https://doi.org/10.1016/j.ejor.2008.11.009}
  {\path{doi:https://doi.org/10.1016/j.ejor.2008.11.009}}.
\newline\urlprefix\url{https://www.sciencedirect.com/science/article/pii/S0377221708009752}

\bibitem{Pairwise2}
Y.~Shiraishi, K.~Fukumizu,
  \href{https://www.sciencedirect.com/science/article/pii/S0925231210003735}{Statistical
  approaches to combining binary classifiers for multi-class classification},
  Neurocomputing 74~(5) (2011) 680--688.
\newblock \href {https://doi.org/https://doi.org/10.1016/j.neucom.2010.09.004}
  {\path{doi:https://doi.org/10.1016/j.neucom.2010.09.004}}.
\newline\urlprefix\url{https://www.sciencedirect.com/science/article/pii/S0925231210003735}

\bibitem{MNIST}
L.~Deng, The mnist database of handwritten digit images for machine learning
  research [best of the web], IEEE Signal Processing Magazine 29~(6) (2012)
  141--142.
\newblock \href {https://doi.org/10.1109/MSP.2012.2211477}
  {\path{doi:10.1109/MSP.2012.2211477}}.

\bibitem{NNPlotter}
H.~Iqbal,
  \href{https://doi.org/10.5281/zenodo.2526396}{Harisiqbal88/plotneuralnet
  v1.0.0} (Dec. 2018).
\newblock \href {https://doi.org/10.5281/zenodo.2526396}
  {\path{doi:10.5281/zenodo.2526396}}.
\newline\urlprefix\url{https://doi.org/10.5281/zenodo.2526396}

\bibitem{Amudhan}
A.~Krishnaswamy-Usha, A.~J. Kearsley, P.~N. Patrone, Analysis of a quadratic
  binary classification scheme using a homotopy-path optimizer (2025).

\bibitem{Homotopy1}
L.~T. Watson, R.~T. Haftka, Modern homotopy methods in optimization, Computer
  Methods in Applied Mechanics and Engineering 74~(3) (1989) 289--305.

\bibitem{Homotopy2}
E.~L. Allgower, K.~Georg, Introduction to Numerical Continuation Methods,
  Society for Industrial and Applied Mathematics, 2003.
\newblock \href {https://doi.org/10.1137/1.9780898719154}
  {\path{doi:10.1137/1.9780898719154}}.

\bibitem{Homotopy3}
B.~Addis, M.~Locatelli, F.~Schoen, Local optima smoothing for global
  optimization, Optimization Methods and Software 20~(4-5) (2005) 417--437.
\newblock \href {https://doi.org/10.1080/10556780500140029}
  {\path{doi:10.1080/10556780500140029}}.

\bibitem{Homotopy4}
D.~M. Dunlavy, D.~P. O'Leary, Homotopy optimization methods for global
  optimization. (12 2005).
\newblock \href {https://doi.org/10.2172/876373} {\path{doi:10.2172/876373}}.

\bibitem{Homotopy5}
Z.~Qiu, L.~Peng, A.~Manatunga, Y.~Guo, A smooth nonparametric approach to
  determining cut-points of a continuous scale, Computational Statistics and
  Data Analysis 134 (2019) 186--210.

\bibitem{Regularized}
N.~Tsoi, K.~Candon, Y.~Milkessa, M.~V'azquez,
  \href{https://api.semanticscholar.org/CorpusID:238226544}{An end-to-end
  approach for training neural network binary classiﬁers on metrics based on
  the confusion matrix}, 2021.
\newline\urlprefix\url{https://api.semanticscholar.org/CorpusID:238226544}

\bibitem{Disc1}
A.~Ng, M.~Jordan, On discriminative vs. generative classifiers: A comparison of
  logistic regression and naive bayes, Advances in neural information
  processing systems 14 (2001).

\bibitem{Disc2}
T.~Jebara, Generative Versus Discriminative Learning, Springer US, Boston, MA,
  2004, pp. 17--60.

\end{thebibliography}

\end{document}